\documentclass[letterpaper]{article}
\usepackage{natbib}
\newcommand{\citecont}[1]{\citealp{#1}}
\usepackage{amssymb}
\usepackage{amsmath}
\usepackage{amsthm}
\usepackage{bbm}
\usepackage{comment}
\usepackage[utf8]{inputenc} 
\usepackage[T1]{fontenc}    
\usepackage{hyperref}       
\usepackage{url}            
\usepackage{booktabs}       
\usepackage{amsfonts}       
\usepackage{nicefrac}       
\usepackage{microtype}      
\usepackage[ruled,linesnumbered,vlined]{algorithm2e}
\usepackage{bm}
\DeclareMathOperator*{\argmax}{\arg\,\max}
\DeclareMathOperator*{\argmin}{\arg\,\min}
\usepackage{symbols}
\usepackage{amsthm}
\usepackage[dvipdfmx]{graphicx}
\usepackage{color}
\newtheorem{prop}{Proposition}
\newtheorem{thm}{Theorem}
\newtheorem{lem}{Lemma}
\newtheorem{cor}{Corollary}

\newcommand{\edited}[1]{#1}

\author{Liyuan Xu${}^{\dagger\ddagger}$ \and Junya Honda${}^{\dagger\ddagger}$\and Masashi Sugiyama${}^{\ddagger\dagger}$ \\
\vspace{3mm}\\
$\dagger$\,:The University of Tokyo \hspace{3mm} $\ddagger$\,:RIKEN
}

\title{Dueling Bandits with Qualitative Feedback}
\begin{document}

\maketitle
\begin{abstract}
 We formulate and study a novel multi-armed bandit problem
called the \emph{qualitative dueling bandit (QDB)} problem, where an agent
observes not numeric but qualitative feedback by pulling each arm.
We employ the same regret as the \emph{dueling bandit} (DB) problem where the duel is carried out by comparing the qualitative feedback. Although we can naively use classic DB algorithms for solving the QDB problem, this reduction significantly worsens the
performance---actually, in the QDB problem, the probability that one
arm wins the duel over another arm can be \emph{directly} estimated
without carrying out actual duels.
In this paper, we propose such direct algorithms for the QDB problem.
Our theoretical analysis shows that the proposed algorithms significantly outperform DB algorithms by incorporating the
qualitative feedback, and experimental results also demonstrate vast improvement over the
existing DB algorithms.
\end{abstract}

\section{Introduction}
The stochastic multi-armed bandit (MAB) problem is a sequential decision-making problem that an agent repeatedly chooses one option from $K$ alternatives, which are often called arms. At each round, the agent receives a random reward that depends on the arm being selected, and the goal is to maximize the cumulative reward. This problem has been extensively studied for many years, both from theoretical and practical aspects. Numerous algorithms has been  proposed for the problem \cite{Thompson1933,Auer2003}, and applied to various fields including the design of clinical trial \cite{villar2015}, economics \cite{Rothschild1974}, and crowdsourcing \cite{Zhou2014}.

The dueling bandit (DB) problem \cite{Yue2012} is a variant of the MAB problem, where an agent only observes the result of the ``duel'', a noisy comparison between the selected two arms. While the MAB problem assumes that the feedback is numeric, the DB problem only assumes that the arms are comparable based on the feedback. Therefore, it is useful for the case where the numeric feedback is not available, such as information retrieval and clinical trial, in which the feedback is qualitative by nature.


Even in the case where the numeric feedback is not available, we may still have access to qualitative feedback. For example, in information retrieval, users might report the relevance of a document returned by a system on a scale of {\sf ``Irrelevant''---``Partially Relevant''---``Relevant''}.  In such a situation, we can consider a special kind of the DB problem first introduced by \citet{Busa-Fekete2013}, which we call the {\it qualitative DB (QDB)} problem. 



In the QDB problem, an agent pulls one arm at each round and observes qualitative feedback. Although the duel is not conducted explicitly in the QDB problem, the algorithm is evaluated based on the same criterion as the DB problem. Here, the probability of an arm winning a duel with another arm corresponds to the probability of the arm getting higher qualitative feedback than the other.  Therefore, we can adapt any algorithms for the DB problem to the QDB problem by converting the feedback in every two rounds into the result of one duel.

However, this reduction significantly worsens the performance because, in the QDB problem, the winning probability can be calculated from the estimated feedback distributions. \citet{Busa-Fekete2013} also partially considered this problem, and they succeeded in improving the performance of the classic DB algorithms by constructing a tight confidence bound. However, they still use the same exploration strategy as the classic DB algorithm.  In this paper, we show that we can further improve the performance by designing a special exploration strategy for the QDB problem.

Several definitions of the ``best arm'' have been proposed for the DB problem.  In this paper, we consider two types of winners, the {\it Condorcet winner} and the {\it Borda winner}, both of which are defined in Section~\ref{sec:problem-formulation}, and we propose algorithms for each winner.
The proposed algorithms are inspired by algorithms in the MAB, namely Thompson sampling \cite{Thompson1933} and the upper confidence bound (UCB) algorithm \cite{Auer2003}. Interestingly, the algorithm based on Thompson sampling, one of the most popular algorithms for the MAB problem, only works for the criterion of the Condorcet winner and suffers polynomial regret in a specific instance in the criterion of the Borda winner. 


The paper is structured as follows. After discussing the related work in Section 2, we formulate the QDB problem in detail in Section 3. We introduce the two formulations of the QDB problem and propose algorithms for these problems in Sections 4 and 5. Lastly, we show the empirical results for the information retrieval setting in Section 6. 
\section{Related Work}
There are two lines of researches that relate with the QDB problem. The first is the DB problem \cite{Yue2012},  which is the MAB problem \edited{with the feedback given as a form of noisy comparison between two arms.} Many researches have been conducted for this problem and some of them discuss specific comparison models. For example, \citet{Hofmann2011} discussed the case where the duel is carried out by the interleaved comparison with some user model, and \citet{Yue2012} introduced Bradley-Terry model. Among them, several models involve random variables corresponding to the utilities associated with arms, and the result of the duel is determined by the order of such variables. For example, Gaussian model \cite{Yue2012} is the case where the random variables follows a Gaussian distribution, and \citet{Busa-Fekete2013} considered the case where the random variables on a partially ordered set as in the QDB problem.


In the DB problem, the definition of the ``best arm'' is no longer straightforward because there may exist cyclic preference. Although early work of the DB assumes the total order on arms to ensure the existence of the maximal element, recent work has mainly sought to design algorithms for finding the {\it Condorcet winner} \cite{Urvoy2013}, which is the arm that wins over all the other arms with probability larger than or equal to $1/2$. This definition can be regarded as a natural generalization of the maximal element, since the Condorcet winner reduces to the maximal element when the total order exists.  A number of algorithms have been proposed for the Condorcet winner \cite{Urvoy2013,Komiyama2015,Wu2016}. 

A drawback of this formulation is that the Condorcet winner does not always exist. In such cases, we may introduce other notions of the winners, such as the {\it Borda winner} \cite{Urvoy2013} and the {\it Copeland set} \cite{Zoghi2015}.  \citet{Ramamohan2016} introduced numerous notions of the winners other than the Condorcet winner. 

The other line of the related work is qualitative multi-armed bandit (QMAB) problem \cite{Szorenyi2015}, in which an agent also receives qualitative feedback according to the chosen arm. The difference between the QDB problem and the QMAB problem is that the QDB problem handles the winners defined in the classic DB problem, while the QMAB problem introduces its own definition of a ``winner'', which is defined as the arm with the highest $\tau$-quantile of the feedback distribution for $\tau \in (0,1)$.

This definition is, however, sometimes problematic since it ignores the difference in the feedback distribution below the $\tau$-quantile.  Let us consider the case that we have two types of medicines, A and B, and want to figure out which has less side effect. Then, we can perform clinical trials and obtain feedback from patients about the severeness of side effects. 

Assume that the feedback is reported on the scale of {\sf ``No side effect''---``Moderate''---``Severe''} and the true probabilities of getting each feedback are shown in Table~\ref{table:table1}.
Then, we can clearly conclude that medicine A is more preferable since it has a less probability of having a severe side effect, and in fact, medicine A becomes the winner in the formulation of the QDB problem. However, the QMAB problem regards these medicines equally good unless $\tau \leq 0.005$ since the $\tau$-quantile feedback is the same. Nevertheless, setting $\tau \leq 0.005$ is almost impossible in practice since we do not have access to the true probabilities beforehand. 

On the other hand, the definitions of winners considered in the QDB problem are well-studied in the context of voting theory (see \citet{Charon2010}, for a survey), and they dot not have any hyper-parameter to define the problem itself. This makes our algorithms more applicable to the real-world problems.

\begin{table}[t]
\caption{The instance that requires a careful choice of $\tau$ in the QMAB problem.}
   \label{table:table1}
  \begin{tabular}{cccc}
   &{\sf No side effect}&{\sf Moderate}&{\sf Severe}\\\hline
   Medicine A& 0.995 &0.003 & 0.002\\\hline
   Medicine B& 0.995 &0.002 & 0.003
   \end{tabular}
\end{table}

\section{Problem Formulation}\label{sec:problem-formulation}
We formulate the QDB problem in this section. As in the MAB problem, we consider $K$ arms associated with feedback distributions $\nu_1,\dots,\nu_K$, and at each round $t$, the agent chooses one arm $a_t\in[K] = \{1,\dots,K\}$ and receives feedback $r_t$ sampled from distribution $\nu_{a_t}$.
While the MAB problem assumes $\{\nu_i\}$ to be distributions on real values, the QDB considers qualitative feedback which corresponds to the case where $\{\nu_i\}$ are the distributions on the totally ordered set $(\mathcal{L},\preceq)$, where $\mathcal{L}$ is the set of possible feedback and $\preceq$ denotes a total order between feedback. For simplicity, we assume that $\mathcal{L} = [L]$ and total order $\preceq$ corresponds to  order relation $\leq$, which means $1 \preceq 2 \dots \preceq L$. Thus, distributions $\{\nu_i\}_{i=1}^K$ are all categorical, supports of which are $[L]$. Note that even though the rewards $r_t$ are nominal for notational simplicity, the sum of the feedback has no meaning in the QDB setting. 

The QDB problem aims to minimize the same regret as the classic DB problem, which is defined based on pairwise comparison. Following early work \cite{Busa-Fekete2013}, we characterize $\mu_{i,j}$, the probability of arm $i$ winning over arm $j$, as
\begin{align*}
\mu_{i,j} = \prob{X_i > X_j} + \frac12 \prob{X_i = X_j},
\end{align*}
where $X_i$ and $X_j$ are mutually independent random variables following distributions $\nu_i$ and $\nu_j$, respectively. 

We consider two types of winners in this paper. The first one is the Condorcet winner, which is the arm that wins all the other arms with probability larger than or equal to $1/2$. Formally, arm $i^*$ is the Condorcet winner if $\mu_{i^*,j} \geq 1/2$ for all $j\neq i^*$. We denote the Condorcet winner as $a^*_{\Condorcet}$, and the goal of the QDB problem when employing the Condorcet winner is to minimize the following regret:
\begin{align*}
    R_T^\Condorcet = \sum_{t=1}^T \Delta_{a_t}^\Condorcet,
\end{align*}
where $\Delta_i^\Condorcet = \mu_{a^*_\Condorcet,i} - 1/2$. 

The second winner is the Borda winner, which is the arm with the largest Borda score, the average of the winning probabilities against other arms. Formally, the Borda score $B_i$ for arm $i$ is defined as 
\begin{align*}
    B_i = \frac{1}{K-1} \sum_{j \neq i} \mu_{i,j},
\end{align*}
and thus the Borda winner $a^*_\Borda$ is $a^*_\Borda = \argmax_{i\in[K]} B_i$. The regret to minimize in this case is formulated as 
\begin{align*}
    R_T^\Borda = \sum_{t=1}^T \Delta_{a_t}^\Borda,
\end{align*}
where $\Delta_i^\Borda = B_{a^*_\Borda} - B_i$.

The QDB problem can be solved by any algorithm for the classic DB since the same regret is used between them. Algorithms for the DB problem specify two arms $(i,j)$ to compare  at each round and receive a result of the noisy comparison generated from $\mathrm{Ber}(\mu_{i,j})$, where $\mathrm{Ber}(p)$ is the Bernoulli distribution with success probability $p$. This comparison can be simulated in the QDB problem as follows: We observe $X_i$ and $X_j$ by pulling both arms and return which $X_i > X_j$ or $X_i < X_j$ occurred with the ties broken at random. 

However, in the QDB problem, we can directly estimate $\mu_{i,j}$ from the feedback distribution of each arm, which significantly enhances exploration. Considering that $\{\nu_i\}$ are all categorical distributions on $[L]$, we have another representation for $\mu_{i,j}$ given by
\begin{align*}
\mu_{i,j} = \sum_{k=1}^L \arm{P}{i}_k \left(\sum_{l=1}^k \arm{P}{j}_l - \frac12 \arm{P}{j}_k \right),
\end{align*}
where $\arm{P}{i}_k = \prob{X_i = k}$. Let $\mathcal{P}_L$ be the probability simplex $\mathcal{P}_L = \{\vec{x}\in [0,1]^L | \sum_{i=1}^L x_i = 1\}$, and we define function $\mu:\mathcal{P}_L \times \mathcal{P}_L \to [0,1]$ as
\begin{align}
\mu(\vec{x},\vec{y}) = \sum_{k=1}^L x_k \left(\sum_{l=1}^k y_l - \frac12 y_k\right). \label{eq:mu-def}
\end{align}
Hence, $\mu(\vec{P}^{(i)},\vec{P}^{(j)}) = \mu_{i,j}$ for $\vec{P}^{(i)} = (\arm{P}{i}_1,\dots,\arm{P}{i}_L)^\top$. 

\section{Qualitative Dueling Bandit with the Condorcet Winner}
In this section, we propose an algorithm for the QDB problem with the Condorcet winner. The algorithm is called {\it Thompson Condorcet sampling}, which is based on Thompson sampling \cite{Thompson1933}, an algorithm famous for its good performance in the standard MAB problem and wide applicability to many other problems.

This algorithm maintains Bayesian posterior distributions of $\vec{P}^{(i)}$ defined in Section~\ref{sec:problem-formulation}. We employ the Dirichlet distribution $\Dir(\alpha_1,\dots,\alpha_L)$ as the prior distribution, the probability density function of which is
\begin{align*}
	f(\vec{\theta}; \alpha_1,\dots,\alpha_L) &= \frac{\Gamma\left(\sum_{i=1}^L \alpha_i\right)}{\prod_{i=1}^L \Gamma(\alpha_i)} \prod_{i=1}^L \theta_i^{\alpha_i-1},
\end{align*}
where $\Gamma(x)$ is the gamma function. 

Having Dirichlet distributions as priors is a convenient choice when observations are sampled from a categorical distribution. Let $\vecarm{C}{i}(t) = (\arm{C}{i}_1(t), \dots,\arm{C}{i}_L(t))$ be the vector representing the observation until the $t$-th round, where $\arm{C}{i}_k(t) \in \{0,1,\dots\}$ represents the number of times that the feedback $k\in[L]$ is observed when arm $i\in[K]$ is pulled. If we employ the prior distribution as $\Dir(1,\dots,1)$, then the posterior distribution given observations $\vecarm{C}{i}(t)$ is $\Dir(C_1^{(i)}(t)+1,\dots,C_L^{(i)}(t)+1)$. For notational simplicity, we sometimes denote $\vecarm{C}{i}(t)$ as $\vecarm{C}{i}$ when the round $t$ is obvious from the context.

\begin{algorithm}[t]
\SetKwInput{Input}{Input}\SetKwInOut{Output}{Output}
\SetKw{KwInit}{Initialize}
\label{Thompson_condorcet}
Set $\vecarm{C}{i} = \vec{0}$ for all $i \in [K]$\;
Pull all arms $t_0$ times, update $\vecarm{C}{i}$\;
\ForEach{$t=Kt_0,Kt_0+1,\dots, T$}{
    For each arm $i$, sample $\vecarm{\theta}{i}$ from $\Dir(C_1^{(i)} + 1, \dots, C_L^{(i)} + 1)$\; \label{line:thomson_condorcet_algo:sample}
    \uIf{$\exists i: \mu(\vecarm{\theta}{i},\vecarm{\theta}{j}) \geq \frac12 $ for all $j\in[K]$}{
        Pull arm $a_t = i$, observe reward $r_t$\;
        Set $C_{r_t}^{(a_t)} \leftarrow C_{r_t}^{(a_t)} + 1$\;
    }
    \Else{
      \tcp{If the Condorcet winner does not exist, sample $\{\vecarm{\theta}{i}\}_{i=1}^K$ again.}
      Goto Line~\ref{line:thomson_condorcet_algo:sample}\;
    }
}
\caption{Thompson Condorcet sampling}
\end{algorithm}

The entire algorithm is shown in Algorithm~\ref{Thompson_condorcet}. At each round $t$, the algorithm samples  $\vec{\theta}^{(i)}$ from posterior distributions of $\vec{P}^{(i)}$, and pulls the Condorcet winner in $(\vec{\theta}^{(1)}, \dots, \vec{\theta}^{(L)})$. If the Condorcet winner does not exist, the algorithm samples $(\vec{\theta}^{(1)}, \dots, \vec{\theta}^{(L)})$ again.

Let $\vecarm{P^*}{i}$ be 
\begin{align*}
    \vecarm{P^*}{i} = \argmin_{\vec{P}\in\mathcal{P}_L} \KL{\vecarm{P}{i}}{\vec{P}}\quad \mathrm{s.t.}\, \mu(\vec{P}, \vecarm{P}{a^*_\Condorcet})\geq \frac12
    \end{align*}
for Kullback-Leibler (KL) divergence $\KL{\vec{x}}{\vec{y}} = \sum_{i=1}^L x_i\log{\frac{x_i}{y_i}}$.
Then, the regret of Thompson Condorcet sampling is bounded as follows.
\begin{thm}\label{thm:thompson-condorcet-regret}
    If the Condorcet winner exists and $t_0$ is set larger than some constants specified in \eqref{eq:t_0-def-1} and \eqref{eq:t_0-def-2}, the regret of Thompson Condorcet sampling is bounded by
    \begin{align}
    \expect{R^{\Condorcet}_T} &\leq \sum_{i=1}^K(1+\varepsilon)\frac{\Delta^\Condorcet_i}{\KL{\vecarm{P}{i}}{\vecarm{P^*}{i}}}\log T \notag\\
    &\quad~~~~~~~ + O\left((\log\log T)^2\right) + O\left(\frac{1}{\varepsilon^{2L}}\right) \label{eq:simple-thompson-condorcet-bound}
    \end{align}
    for any sufficiently small $\varepsilon > 0$.
\end{thm}

The proof is given in Appendix~\ref{sec:Proof-of-thompson-condorcet-regret}, where the detailed condition on $t_0$ and the precise form of the bound is also provided. From the precise form of \eqref{eq:simple-thompson-condorcet-bound} that can be found in \eqref{eq:thompson-precise-regret} in Appendix~\ref{sec:Proof-of-thompson-condorcet-regret}, one can see that this regret bound grows exponentially with the number of arms $K$. However, this is not the inherent limitation of the Thompson Condorcet sampling but the artifact of pursuing the optimal asymptotic dependence on $O(\log T)$. As we will show in Section~\ref{sec:Experiments}, this exponential increase in the regret does not occur in pracitice, and the algorithm works well for relatively large $K$.


The regret bound has a similar form to the information theoretic lower bound in the MAB problems for 
multi-parameter models \cite{Burnetas1996}. Note that considering distributions $\vecarm{P^*}{i}$ is essential in these case, whereas they are replaced with the distribution of the optimal arm in the regret bound of Thompson sampling in the MAB problem with the Bernoulli model given by \citet{Agrawal2012}. For example, when $\vecarm{P}{a^*_\Condorcet} = (\varepsilon, 1-2\varepsilon,\varepsilon)^\top$ and $\vecarm{P}{i} = (0.5, 0.1,0.4)^\top$, we have $\KL{\vecarm{P}{i}}{\vecarm{P^*}{i}}/\KL{\vecarm{P}{i}}{\vecarm{P}{a^*_\Condorcet}} \to 0$ as $\varepsilon\to 0$.


Theorem~\ref{thm:thompson-condorcet-regret} suggests the possibility of Thompson Condorcet sampling performing drastically better than the case when we apply classic DB algorithms for the QDB problem in the way discussed in Section~\ref{sec:problem-formulation}. The regret lower bound of such direct applications immediately follows from the lower bound for the classic DB problem given by \citet{Komiyama2015}.

\begin{prop}[Adapted from \citecont{Komiyama2015}] \label{prop:regret-lower-bound-condorcet}
	When we apply any consistent algorithms for the DB problem to the QDB problem, we have
	\begin{align}
		\liminf_{T\to \infty} \frac{\expect{R^\Condorcet_T}}{\log T} \geq \sum_{i \neq a^*_\Condorcet} \min_{j: \mu_{i,j} < \frac12} \frac{\Delta_i^\Condorcet + \Delta_j^\Condorcet}{d(\mu_{i,j},\frac12)}, \label{eq:lower-bound-condorcet}
	\end{align}
	where $d(x,y) = x\log\frac{x}{y} + (1-x)\log\frac{1-x}{1-y}$.
\end{prop}
From the upper bound given in Theorem~\ref{thm:thompson-condorcet-regret}, we have
\begin{align*}
	\lim_{T\to\infty} \frac{\expect{R^\Condorcet_T}}{\log T}  \leq (1+\varepsilon)\sum_{i \neq a^*_\Condorcet} \frac{\Delta^\Condorcet_i}{\KL{\vecarm{P}{i}}{\vecarm{P^*}{i}}},
\end{align*}
which can be arbitrarily smaller than \eqref{eq:lower-bound-condorcet} as stated in the next lemma.

\begin{lem}\label{lem:KL-bigger-mu}
Assume that $a^*_\Condorcet \neq 1$. For any fixed $0<\varepsilon<1/(4-4\log2)$, there exist $\vecarm{P}{a^*_\Condorcet},\vecarm{P}{1} \in \mathcal{P}_2$ such that 
\begin{align}
\frac{d(\mu(\vecarm{P}{a^*_\Condorcet},\vecarm{P}{1}),1/2)}{\KL{\vecarm{P}{1}}{\vecarm{P^*}{1}}} \leq \varepsilon.
\end{align}
\end{lem}

The proof can be found in Appendix~\ref{sec:Proof-of-thompson-condorcet-regret}. From Lemma~\ref{lem:KL-bigger-mu}, we can say that there exists the case where Thompson Condorcet sampling can perform arbitrarily better than the direct application of any algorithms in the DB. This implies that the algorithm successfully incorporates the qualitative information to reduce the regret in the DB.

\section{Qualitative Dueling Banidt with the Borda Winner} \label{sec:Borda_winner}
In this section, we study two algorithms for the QDB problem with the Borda winner,  the one based on the Thompson sampling called {\it Thompson Borda sampling} and the other based on the UCB algorithm \cite{Auer2003} called {\it Borda-UCB}. In spite of the success of Thompson Condorcet sampling, our theoretical analysis reveals that Thompson Borda sampling can have polynomial regret in some setting. On the other hand, Borda-UCB achieves logarithmic regret, which matches the regret lower bound of the classic DB problems.
 
 \begin{algorithm}[t]
\SetKwInput{Input}{Input}\SetKwInOut{Output}{Output}
\SetKw{KwInit}{Initialize}
\label{alg:Thompson_borda}
Set $\vecarm{C}{i} = \vec{0}$ for all $i \in [K]$\;
Pull all arms $t_0$ times, update $\vecarm{C}{i}$\;
\ForEach{$t=1,\dots,$}{
    For each arm $i$, sample $\vecarm{\theta}{i}$ from $\Dir(C_1^{(i)} + 1, \dots, C_L^{(i)} + 1)$\; 
    $B_i \leftarrow \frac{1}{K-1} \sum_{j\neq i} \mu(\vecarm{\theta}{i},\vecarm{\theta}{j})$\;
    Pull arm $a_t = \argmax_{i\in[K]} B_i$\;
    Observe $r_t$ and set $C_{r_t}^{(a_t)} \leftarrow C_{r_t}^{(a_t)} + 1$\;
}
\caption{Thompson Borda sampling}
\end{algorithm}

Thompson Borda sampling given in Algorithm~\ref{alg:Thompson_borda} is similar to Thompson Condorcet sampling. The only difference is that Thompson Borda sampling pulls the Borda winner in samples $(\vecarm{\theta}{1},\dots,\vecarm{\theta}{L})$. Since there always exists the Borda winner for any samples $(\vecarm{\theta}{1},\dots,\vecarm{\theta}{L})$, thus we do not need resampling. Although it is works surprisingly well empirically as we will see in Section~\ref{sec:Experiments}, we prove that it suffers from polynomial regret in the worst case.

\begin{thm}\label{thm:thompson-fails-borda}
Assume that there are $K=3$ arms such that arm $1$ is the Borda winner. Then, there exists $\vecarm{P}{1},\vecarm{P}{2},\vecarm{P}{3}\in\mathcal{P}_L$ such that under Thompson Borda sampling with $\vecarm{\theta}{1} = \vecarm{P}{1}$, $\vecarm{\theta}{2} = \vecarm{P}{2}$, and $\vecarm{\theta}{3} \sim \Dir(\arm{C}{3}_1+1,\dots,\arm{C}{3}_L+1)$, the statement 
\[\liminf_{T\to\infty} \frac{\expect{R_T^{\Borda}}}{T^\eta} = \xi\]
holds for some constants $\xi,\eta > 0$. 
\end{thm}

The proof can be found in Appendix~\ref{sec:Proof-of-thompson-fails-borda}. The situation considered in Theorem~\ref{thm:thompson-fails-borda} may be somewhat unrealistic since we assume that $\vecarm{P}{1}$ and $\vecarm{P}{2}$ are known beforehand. However, we will show by an experiment that Thompson Borda sampling actually suffers from the polynomial regret without such an assumption in Section~\ref{sec:Experiments}.

Another proposed algorithm, Borda-UCB, is based on the UCB algorithm \cite{Auer2003}, which is shown in Algorithm~\ref{alg:Borda_UCB}. As in the original UCB algorithm, we consider the upper confidence bound $\hat{B}_i + \beta_i$ for each arm $i\in[K]$, where $\hat{B}_i$ is an estimated Borda score, and $\beta_i$ is the width of the confidence interval controlled by a positive parameter $\alpha$. Let $i_{\text{UCB}}$ be the arm with the largest upper confidence bound. While the original UCB algorithm always pulls the arm with the largest upper confidence bound, Borda-UCB pulls all arms that do not belong to $i_{\text{Count}}$, the set of arms that were pulled the most, if $i_{\text{UCB}}$ does not belong to  $i_{\text{Count}}$. This exploration strategy reflects the fact that we have to estimate all feedback distributions accurately in order to have the precise estimation of the Borda score.

\begin{algorithm}[t]
\SetKwInput{Input}{Input}\SetKwInOut{Output}{Output}
\SetKw{KwInit}{Initialize}
\label{alg:Borda_UCB}
Set $\vecarm{C}{i} = \vec{0}$ for all $i \in [K]$ and $N_i = 0$\;
Pull all arms $\tau$ times and get initial estimations\;
\While{$t \leq T$}{
    $\hat{\vec{P}}^{(i)} \leftarrow \vecarm{C}{i}/N_i$ for each arm $i\in[K]$\;
    $\hat{B}_{i} \leftarrow \frac{1}{K-1}\sum_{k\in[K]\backslash\{i\}} \mu(\hat{\vec{P}}^{(i)}, \hat{\vec{P}}^{(k)})$\label{eq:estimate-Borda-score}\;
    $\gamma_i \leftarrow \sqrt{\frac{\alpha\log t}{N_i}}$\;
    $\beta_{i} \leftarrow \gamma_i + \frac{1}{K-1}\sum_{k\in[K]\backslash\{i\}} \gamma_k$\;\label{eq:Borda-confidence-width}
    $i_\UCB \leftarrow \argmax_{i\in[K]} B_i + \beta_i$\;
    $i_\Count \leftarrow \{i \in [K] | N_i = \max_{j\in[K]} N_j\}$\;
    \uIf{$i_\UCB \in i_\Count$}{
        Pull arm $a_t = i_\UCB$, observe reward $r_t$\; \label{line:UCB-exploit}
        $N_i\leftarrow N_i + 1,\, \arm{C}{r_t}_{a_t}\leftarrow \arm{C}{r_t}_{a_t} + 1$\;
    }
    \Else{
        Pull all arms in $[K]\backslash i_\Count$\; \label{line:UCB-explore}
        Update $N_i$ and $\arm{C}{i}_k$\;
    }
    
}
\caption{Borda-UCB}
\end{algorithm}

The regret of Borda-UCB is bounded as follows.

\begin{thm}\label{thm:UCB-borda-regret}
    Assume that $\alpha$ is set as
    \begin{align*}
        \alpha = \max\left(2,\frac{3(1+3\varepsilon')^2}{2(1-\varepsilon')^2}\left(\frac{K-1}{K-2}\right)^2\right)
    \end{align*}
     for arbitrarily taken $\varepsilon'>0$. Then, for any $\varepsilon>0$, the regret of Borda-UCB is bounded as
    \begin{align*}
        \expect{R_T^{\Borda}} \leq \Delta^\Borda_{\mathrm{all}}\left(\frac{4\alpha}{(\Delta^\Borda_{\mathrm{min}}-2\varepsilon)^2}\log T +C_\varepsilon + C_{\varepsilon'}\right)
    \end{align*}
    for some constants $C_\varepsilon = O\left(\frac{1}{\varepsilon^2}\right), C_{\varepsilon'} = O\left(\frac{1}{(\varepsilon')^2}\right)$, where  $\Delta^\Borda_{\mathrm{all}} =\sum_{i\neq a^*_\Borda} \Delta^\Borda_i$ and  $\Delta^\Borda_{\mathrm{min}} = \min_{i\neq a^*_\Borda} \Delta^\Borda_i$.
\end{thm}

  The proof is presented in Appendix~\ref{sec:Proof-of-Borda-regret}, where the explicit forms of $C_\varepsilon$ and $C_\varepsilon'$ are also provided.
The regret bound in Theorem~\ref{thm:UCB-borda-regret} is simplified to $O(K \Delta^{-2}\log T)$ when $\Delta^\Borda_i=\Delta$ for all $i \neq i^*$, while the regret of the original UCB algorithm is $O(K \Delta^{-1}\log T)$ \cite{Auer2003}, which is smaller by $O(1/\Delta)$.
 However, this difference is inevitable, as proved in the following theorem.


\begin{thm}\label{thm:borda-lower-bound}
	Consider two instances of the QDB problem with $K=3$, in which the feedback distributions of the arms are represented as $\Gamma = (\vecarm{P}{1}_\Gamma,\vecarm{P}{2}_\Gamma,\vecarm{P}{3}_\Gamma)$  and $\Theta = (\vecarm{P}{1}_\Gamma,\vecarm{P}{2}_\Gamma,\vecarm{P}{3}_\Gamma)$. Let $R^\Gamma_T$ and $R^\Theta_T$ be the regret in each instance. Then, there exists a pair of instances $(\Gamma,\Theta)$ that all algorithms which achieve
	\begin{align*}
		\expect{R^\Gamma_T} \leq o(T^{a})
	\end{align*}
for all constant $a>0$ satisfy
	\begin{align*}
		\liminf_{T\to\infty} \frac{\expect{R^\Theta_T}}{\log T} \geq \Omega\left(\frac{1}{(\Delta^\Borda_\mathrm{min})^2}\right),
	\end{align*}
where $\Delta^\Borda_\mathrm{min} = \min_{i \neq a^*_\Borda} \Delta_i^\Borda$ defined on $\Theta$.
\end{thm}

The proof is presented in Appendix~\ref{sec:Proof-of-borda-lower-bound}. This theorem states that if the algorithm achieves sub-polynomial regret for all instances of the QDB problem with the Borda winner, there exists a case where it suffers from $\Omega((\Delta^\Borda_{\text{min}})^{-2}\log T )$ regret. Therefore, we can conclude that the difference in the regret upper-bound between the original UCB and Borda-UCB comes from the characteristic of the QDB problem.

The upper bound in Theorem~\ref{thm:UCB-borda-regret} matches the regret lower bound in the classic DB problem, which is considered in the context of the $\delta$-PAC DB problem \cite{Jamieson2015}. 
The algorithm is called $\delta$-PAC if it finds the Borda winner with failure probability less than $\delta$. We have the following bound of the minimum number of samples required in such $\delta$-PAC algorithms.

\begin{prop}[Theorem 1; \citecont{Jamieson2015}] \label{prop:lower-bound-borda}
	Let $\tau$ be the total number of pulls. If $K \geq 4$ and $3/8 \leq \mu_{i,j} \leq 5/8$ for all $i,j\in[K]$, then  any $\delta$-PAC DB algorithm with $\delta \leq 0.15$ has
	\begin{align*}
		\expect{\tau} \geq \frac1{90}\log\frac{1}{2\delta} \sum_{i\neq a^*_\Borda} \frac{1}{(\Delta^\Borda_i)^2}.
	\end{align*}
\end{prop}

Existing algorithms for the Borda winner \cite{Busa-Fekete2013,Jamieson2015} use a $\delta$-PAC DB algorithm as a sub-routine. They first run such an algorithm with $\delta = 1/T$ and then pulls the estimated Borda winner in the remaining rounds. Therefore, the regret of such algorithms is at least $\Omega((\log T) \sum_{i\neq a^*_\Borda} (\Delta^\Borda_i)^{-2} )$ from Proposition~\ref{prop:lower-bound-borda}, and hence the regret upper bound of Borda-UCB is no worse than this lower bound.

Although we were not able to prove that the regret of Borda-UCB is smaller than the direct application of classic DB algorithms, Borda-UCB performs better than them empirically as we will see in Section~\ref{sec:Experiments}. Furthermore,  Borda-UCB has an another advantage that it does not require to specify $T$. Since existing algorithms run a $(1/T)$-PAC algorithm, it requires the number of rounds $T$ to be known beforehand. However, it is often difficult to guess $T$ beforehand, and thus our algorithms are more useful in practice.

\section{Experiments} \label{sec:Experiments}
We test the empirical performance of the proposed algorithms through experiments based on both synthetic setting and real-world data.  We first conduct the experiments based on the real-world web search dataset that is also used in the previous work. In the experiments, our methods significantly outperform the direct application of the existing algorithms for the classic DB. Then, we show the results of the experiments in a synthetic setting that Thompson Borda sampling has polynomial regret. 

\subsection{Experiments on a Real-World Dataset}

We apply proposed methods to the problem of ranker evaluation from the field of information retrieval, which is used for evaluating the algorithms for the classic DB problem in \citet{Jamieson2015}. The task is to identify the best ranker, which takes a user's search query as input and ranks the documents according to their relevance to that query. 

We used two web search datasets. The first is the MSLR-WEB10K dataset \cite{Qin2010}, which consists of 10,000 search queries over the documents from search results. The data also contains the values of 136 features and a corresponding user-labeled relevance factor on a scale of one to five with respect to each query-document pair. 
The other is the MQ2008 dataset \cite{Qin2013} that contains 46 features and a relevance factor labelled from one to three for each query-document pair. As in \citet{Jamieson2015}, we only consider rankers that use one feature to rank documents. Therefore, the aim of the task is to determine which feature is the most capable of predicting the relevance of query-document pairs.

Although \citet{Jamieson2015} set up the classic DB problem from these datasets, we can naturally formulate the QDB problem as well since we have access to the relevance factors. The qualitative feedback is generated in the following way. At each round, the algorithm selects one ranker, and it ranks the documents for a randomly chosen query. The relevance factor for the top-ranked document is revealed to the algorithm as the qualitative feedback. Therefore, we have $L=5$ in the MSLR-WEB10K dataset and $L=3$ in the  MQ2008 dataset. We compare the regrets of the proposed algorithms to the direct application of the classic DB algorithms, which corresponds to the experiments conducted in \citet{Jamieson2015}. We repeat 100 runs for each instance and the mean of the regret is reported.

\subsubsection{Experiments for Condorcet Winner}
We first show the experimental result of the QDB problem with the Condorcet winner. We compare Thompson Condorcet sampling with RUCB \cite{Zoghi2014}, RMED1, RMED2, RMED2F \cite{Komiyama2015}, which are all promising algorithms proposed for the classic DB problem with the Condorcet winner. We set $t_0 = 10$, and the Figure~\ref{fig:Condorcet} is the experimental result when the number of rankers is $K=5$.

\begin{figure}
    \hspace{-2.7mm}
    \begin{tabular}{cc}
    \begin{minipage}{0.5\linewidth}
    \includegraphics[width = \columnwidth]{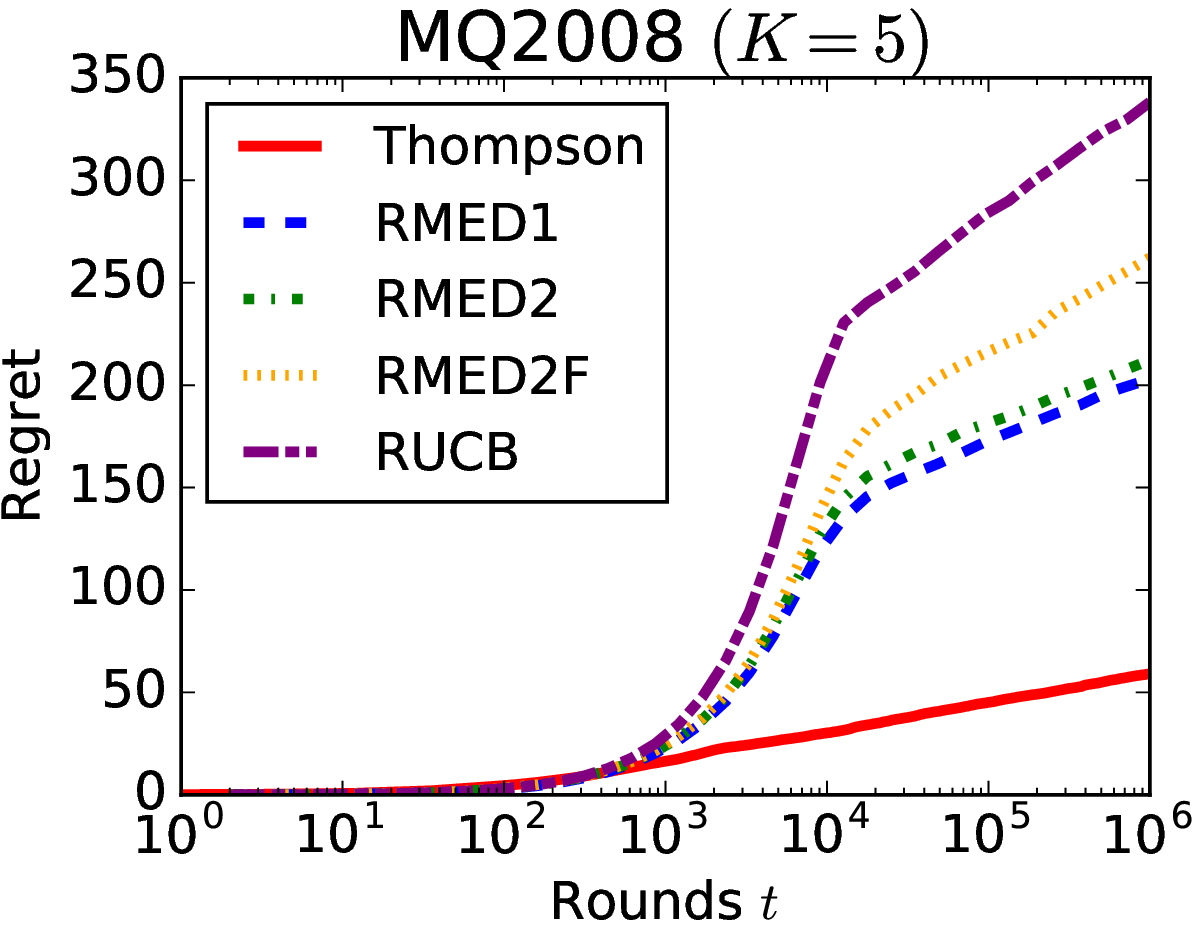}
    \end{minipage}&
    \hspace{-5.6mm}
    \begin{minipage}{0.5\linewidth}
    \includegraphics[width = \columnwidth]{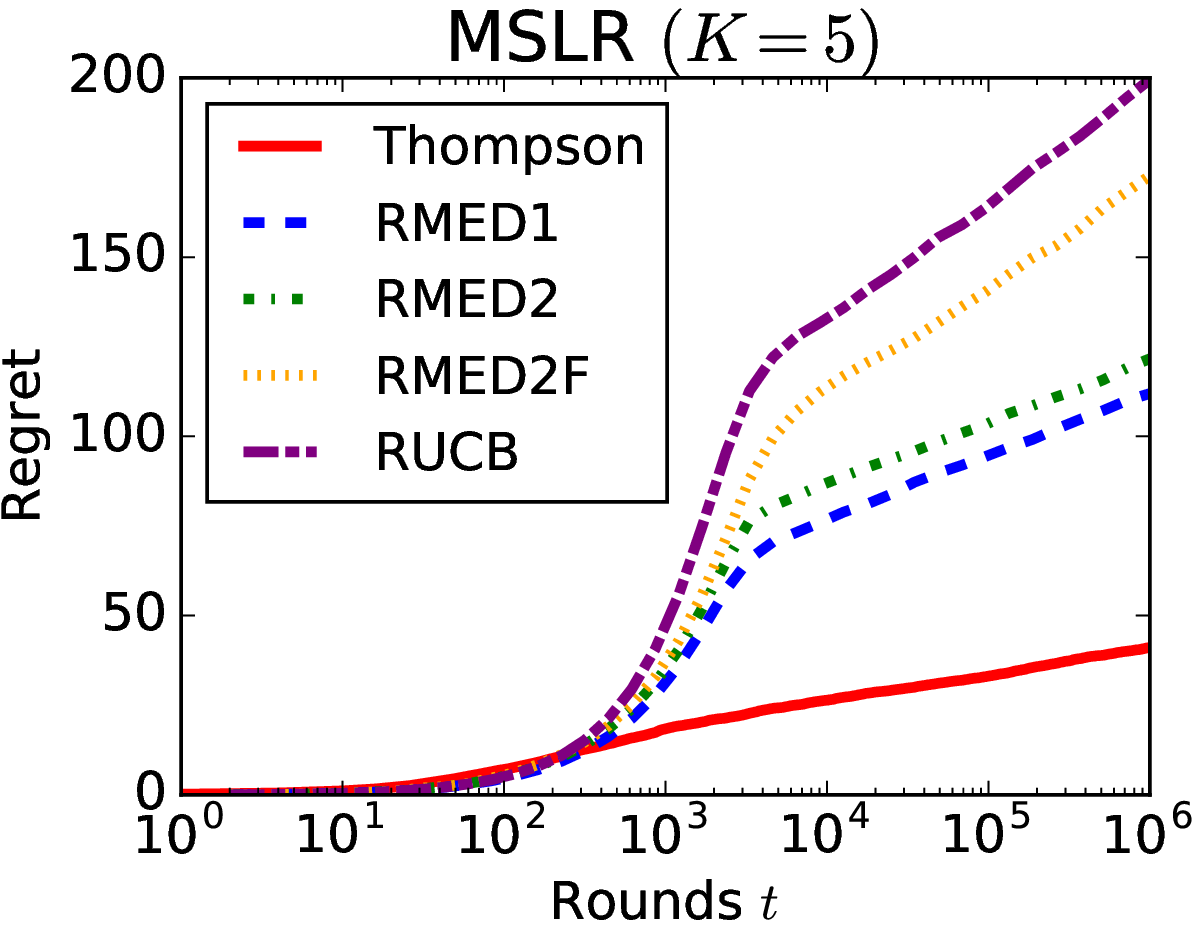}
    \end{minipage}
    \end{tabular}
    \caption{The regret of Thompson Condorcet sampling and other classic DB algorithms.}
    \label{fig:Condorcet}
\end{figure}

Figure~\ref{fig:Condorcet} shows the superiority of Thompson Condorcet sampling. Furthermore, we can observe all existing algorithms incur the large regrets in early rounds while Thompson Condorcet sampling does not.  This is because most algorithms for the DB problem construct a set of candidates for the Condorcet winner and explores it in the first part of the rounds, but Thompson Condorcet sampling conducts exploration and exploitation at the same time and does not require such a set. In this sense, Thompson Condorcet sampling performs more stably than the existing methods.

\begin{figure}
    \centering
    \includegraphics[width=0.7\linewidth]{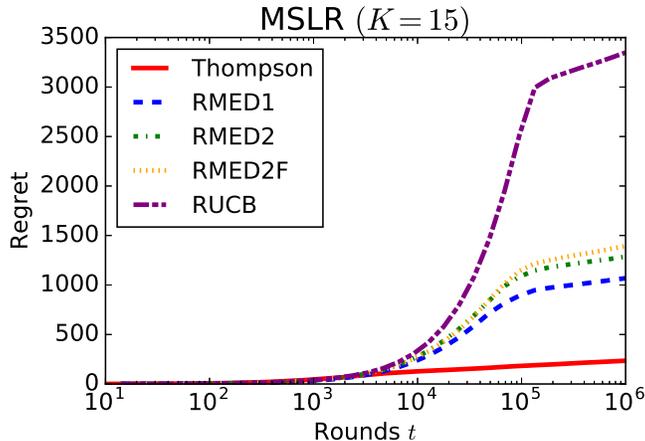}
    \caption{The regret of Thompson Condorcet sampling and other DB algorithms when there are a relatively large number of arms ($K=15$).}
    \label{fig:Condorcet-large-K}
\end{figure}

To see the dependency of the performance of Thompson Condorcet sampling on the number of arms, we tried the setting in which we have a relatively large number of arms. The result is shown in Figure~\ref{fig:Condorcet-large-K}, in which Thompson Condorcet sampling still performs the best among the other classic DB algorithms even though the regret upper-bound proved in Theorem~\ref{thm:thompson-condorcet-regret} grows exponentially with $K$. This result supports the argument that exponential dependency on $K$ is just an artifact of pursuing the best regret bound in the asymptotic case and Thompson Condorcet sampling empirically performs much better than the theoretical analysis.

\subsubsection{Experiments for Borda Winner}
For the Borda setting, we compare our proposed methods, Thompson Borda Sampling and Borda-UCB, with existing classic DB algorithm SSSE \cite{Busa-Fekete2013}. Furthermore, we also conduct a comparison with an extension of SSSE, which we call QSEEE, proposed in \citet{Busa-Fekete2013} to utilize the qualitative feedback explicitly.

\begin{figure}
    \hspace{-2.7mm}
    \begin{tabular}{cc}
    \begin{minipage}{0.5\linewidth}
    \includegraphics[width = \columnwidth]{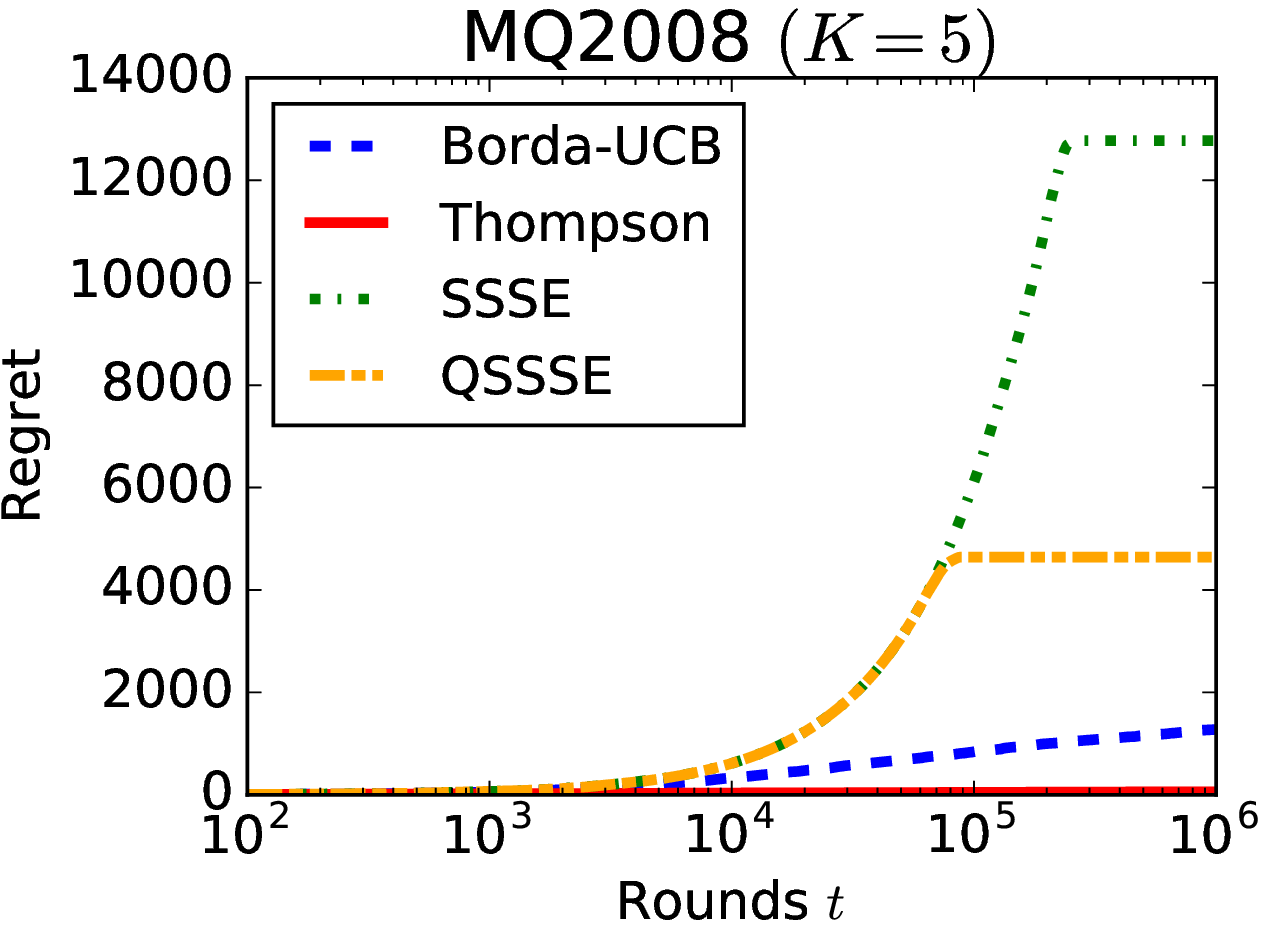}
    \end{minipage}&
    \hspace{-5.6mm}
    \begin{minipage}{0.5\linewidth}
    \includegraphics[width = \columnwidth]{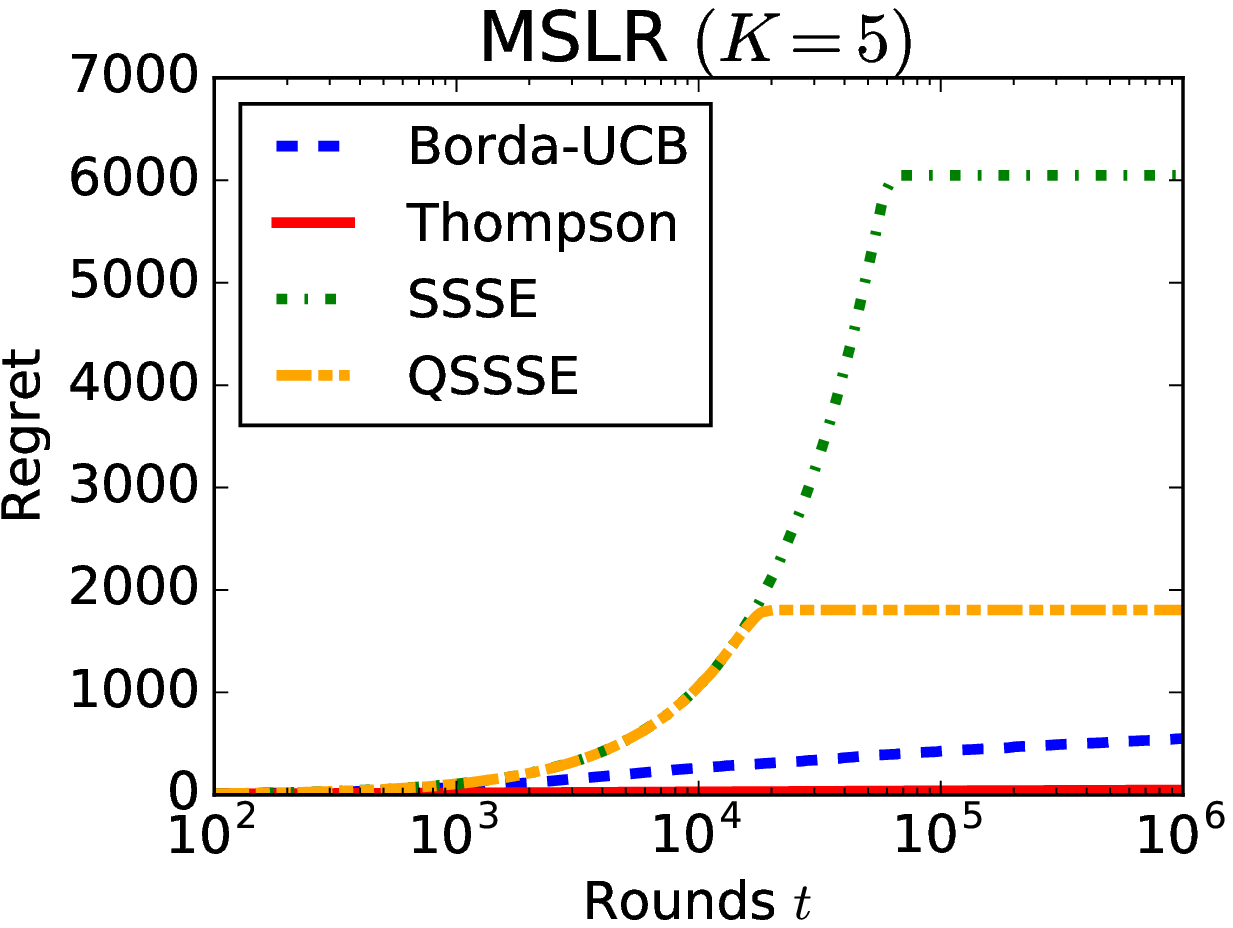}
    \end{minipage}
    \end{tabular}
    \caption{The regret of Thompson Borda sampling and Borda-UCB with other classic DB algorithms.}
    \label{fig:Borda}
\end{figure}

The result is shown in Figure~\ref{fig:Borda}, which shows the superiority of the proposed methods. As in the Condorcet case, SSSE and QSSSE suffer from a large regret in the early stage, while regret always increases logarithmically in the proposed algorithms. This is because existing methods first only explore, while proposing methods always balance exploration and exploitation. Although existing methods achieve zero-regret after the exploration, this does not mean that they perform better than Borda-UCB in $T\to \infty$ since they require longer exploration phase.

Surprisingly,  Thompson Borda sampling works quite well in this setting, even though Theorem~\ref{thm:thompson-fails-borda} states that it has the polynomial regret in the worst case. We suspect it is rare to encounter such a worst case in practice, but the condition for sub-polynomial regret is unknown and left to future work.


\subsection{Experiments on a Synthetic Setting}

Theorem~\ref{thm:thompson-fails-borda} proves that Thompson Borda sampling can incur polynomial regret for some instances, which we confirm through experiments in the following. We set up the instance with $K=3$ and $L=4$, in which each feedback distribution is represented as $\vecarm{P}{1} = ( 0.0,  0.0,  1.0,  0.0)^\top$, $\vecarm{P}{2} = (0.0,  0.5,  0.0,  0.5)^\top$, and $\vecarm{P}{3} = (0.2 ,  0.4,  0.3,  0.1)^\top$.
We repeat running Thompson Borda sampling and Borda-UCB in this instance for 10 times, and the mean of regret is shown in Figure~\ref{fig:Thompson-fails}.

\begin{figure}
    \centering
    \includegraphics[width = 0.7\linewidth]{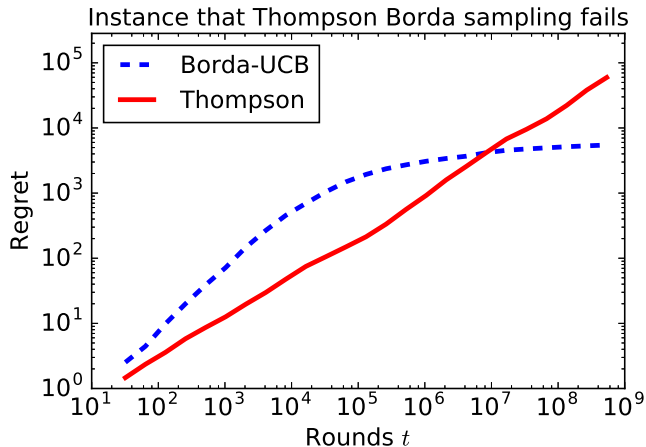}
    \caption{The regret of proposed algorithms in the instance that Thompson Borda sampling suffers the polynomial regret.}
    \label{fig:Thompson-fails}
\end{figure}

From Figure~\ref{fig:Thompson-fails}, we can clearly see that Thompson Borda sampling suffers from polynomial regret, while Borda-UCB still has sub-polynomial regret.  However, it takes many rounds for Borda-UCB to have less regret than Thompson Borda sampling. This is because Thompson Borda sampling explores less than necessary. In early rounds, UCB-Borda pulls arm~3 many times, which is necessary for knowing the Borda winner but incurs large regret. On the other hand, Thompson Borda sampling exploits arms~1 and~2 more, which leads its superior performance in early rounds.

\section{Conclusions}
In this paper, we formulated and studied a novel type of the dueling bandit, called a qualitative dueling bandit. 
In this problem, an agent receives qualitative feedback at each round and aims to minimize the same regret as the classic DB when the duel is carried out based on that feedback.

We considered two notions of winners, the Condorcet winner and the Borda winner. For the Condorcet winner, we proposed an algorithm, called Thompson Condorcet sampling, and we showed that the regret can be arbitrarily smaller than the direct application of the algorithms in classic DB. Thompson Condorcet sampling also exhibited the superior performance in the experiments based on the real-word web search datasets.

For the Borda winner, we studied two algorithms, Thompson Borda sampling and UCB-Borda. Although the theoretical analysis reveals that Thompson Borda sampling can have polynomial regret in some instances, the experiments showed that it performs surprisingly well empirically, especially when the number of rounds is not very large.  On the other hand, we prove the logarithmic regret upper bound for UCB-Borda, which is no worse than the regret lower bound in the classic DB.

As future work, it is important to derive general algorithms that can handle various notions of winners as in \citet{Ramamohan2016}. Another promising direction is to improve the algorithms for the Borda winner and achieve  regret significantly smaller than the classic DB as Thompson Condorcet sampling does in the Condorcet winner case.
 
\section{Acknowledgements}
LX utilized the facility provided by Masason Foundation. JH acknowledges support by KAKENHI 18K17998, and MS acknowledges support by KAKENHI 17H00757.
\bibliography{reference}
\newpage
\appendix
\section{Preliminaries} \label{sec:Preliminaries}
In this section, we introduce the concentration inequalities for multinomial distributions, which are the bounds on how a random variable deviates from the expected value. The first inequality measures deviation terms of the KL-divergence as follows.

\begin{lem}\label{lem:multinomial-bound}
    Let us consider the random variable $(n_1,\dots,n_L)$ sampled from  multinomial distribution $\Multi(n; P_1,\dots,P_L)$. If we denote the true probability as $\vec{P} = (P_1,\dots,P_L)$ and the empirical probability as $\hat{\vec{P}} = (\frac{n_1}{n},\dots,\frac{n_L}{n})$, we have
    \begin{align*}
        \prob{\KL{\hat{\vec{P}}}{\vec{P}} \geq \varepsilon} \leq C_1 n^L \exp(-n\varepsilon)
    \end{align*}
    for any $\varepsilon>0$ and $C_1 = (2\pi)^{-\frac{L-1}{2}}\exp\left(L-\frac56 \right)$.
\end{lem}

\begin{proof}[Proof of Lemma~\ref{lem:multinomial-bound}]
Since
\begin{align}
    \sqrt{2\pi}z^{z-\frac12}\mathrm{e}^{-z} \leq \Gamma(z) \leq \sqrt{2\pi}\mathrm{e}^\frac16 z^{z-\frac12}\mathrm{e}^{-z} \label{eq:starling-inequality}
\end{align}
holds for all $z\geq\frac12$ \citep[Sect. 5.6(i)]{Olver2010}, we have
\begin{align}
    &\prob{\KL{\hat{\vec{P}}}{\vec{P}} \geq \varepsilon} \notag\\
    &= \sum_{\substack{n_1,\dots,n_L: \\ \sum_{i=1}^L n_i = n}} \indi{\KL{\hat{\vec{P}}}{\vec{P}}\geq \varepsilon}  \frac{\Gamma(n+1)}{\prod_{i=1}^L \Gamma(n_i + 1)} \prod_{i=1}^L P_i^{n_i}\notag\\
    &\leq \sum_{\substack{n_1,\dots,n_L: \\ \sum_{i=1}^L n_i = n}} \indi{\KL{\hat{\vec{P}}}{\vec{P}}\geq \varepsilon}  \notag\\
    &\quad~~~~~\frac{\sqrt{2\pi}e^{1/6}(n+1)^{n+1/2}e^{-n-1}}{\prod_{i=1}^L \sqrt{2\pi}(n\hat{P}_i+1)^{n\hat{P}_i+1/2}e^{-n\hat{P}_i-1}} \prod_{i=1}^L P_i^{n\hat{P}_i}\notag\\
    &\leq C_1 \sum_{\substack{n_1,\dots,n_L: \\ \sum_{i=1}^L n_i = n}} \indi{\KL{\hat{\vec{P}}}{\vec{P}} > \varepsilon} \exp\left(F_n(\vec{P},\hat{\vec{P}})\right) \label{eq:multinomial_bound_eq1}
\end{align}
for 
\begin{align}
    C_1&= (2\pi)^{-\frac{L-1}{2}}\exp\left(L-\frac56 \right),\label{eq:c_1-def}\\
    F_n(\vec{P},\hat{\vec{P}}) &= (n+1/2)\log(n+1) + \sum_{i=1}^L (n\hat{P}_i )\log(P)\notag\\
    &\quad \quad - \sum_{i=1}^L (n\hat{P}_i + 1/2)\log(n\hat{P}_i + 1).\label{eq:F-def}
\end{align}
Thus, 
\begin{align*}
    F_n(\vec{P},\hat{\vec{P}}) &= -n \sum_{i=1}^L \hat{P}_i\left(\log \frac{n\hat{P}_i+1}{n+1} - \log P\right) \\
    &\quad~~~~~~ + \frac12\left(\log(n+1) - \sum_{i=1}^L\log(n\hat{P}_i+1)\right)\\
&\leq -n \sum_{i=1}^L \hat{P}_i\left(\log \frac{n\hat{P}_i+1}{n+1/\hat{P}_i} - \log P\right) \\
&\quad~~~~~~+ \frac12\left(\log(n+1) - \sum_{i=1}^L\log(n\hat{P}_i+1)\right)\\
&= -n\KL{\hat{\vec{P}}}{\vec{P}} + \frac12 f(n),
\end{align*}
for $f(n) = \log(n+1) - \sum_{i=1}^L\log(n\hat{P}_i+1)$. Since
\begin{align*}
    f'(n) &= \frac{1}{n+1} - \sum_{i=1}^L \frac{\hat{P}_i}{n\hat{P}_i+1} \leq \frac{1}{n+1} - \sum_{i=1}^L \frac{\hat{P}_i}{n+1}= 0
\end{align*}
and $f(0) = 0$, we have $f(n) \leq 0$. Therefore, from \eqref{eq:multinomial_bound_eq1}, we have
\begin{align*}
    &\prob{\KL{\hat{\vec{P}}}{\vec{P}} \geq \varepsilon} \\
    &\leq C_1 \sum_{\substack{n_1,\dots,n_L: \\ \sum_{i=1}^L n_i = n}} \indi{\KL{\hat{\vec{P}}}{\vec{P}} > \varepsilon} \exp\left(-n\KL{\hat{\vec{P}}}{\vec{P}}\right)\\
    &\leq C_1n^L\exp(-n\varepsilon).
\end{align*}
\end{proof}

We also use the following inequality to handle the deviation measured by the $l_1$ norm.

\begin{lem}[Bretagnolle-Huber-Carol Inequality \cite{Varrt2000}]\label{lem:multinomial-l1-bound}
    For $\vec{P}, \hat{\vec{P}}$ defined in Lemma~\ref{lem:multinomial-bound}, we have
    \begin{align*}
        \prob{\left|\hat{\vec{P}}-\vec{P}\right| \geq \varepsilon} \leq 2^L \exp\left(-\frac{n}{2}\varepsilon^2\right)
    \end{align*}
    for any $\varepsilon>0$, where $|\vec{x}|$ is the $l_1$-norm of vector $|\vec{x}| = \sum_{i=1}^L |x_i|$.
\end{lem}

The last inequality is for the error in the cumulative distribution.

\begin{lem}[Dvoretzky-Kiefer-Wolfowitz inequality (\citecont{Massart1990})] \label{lem:multinomial_cumulative_bound}
    For $\vec{P}, \hat{\vec{P}}$ defined in Lemma~\ref{lem:multinomial-bound}, we have
    \begin{align*}
        \prob{\max_{k\in[L]} \left|\sum_{l=1}^k \hat{P}_l-\sum_{l=1}^k P_l \right| \geq \varepsilon} \leq 2 \exp\left(-2n\varepsilon^2\right),
    \end{align*}
    for any $\varepsilon>0$.
\end{lem}

Next, we introduce the concentration inequality for the Dirichlet distribution.

\begin{lem}\label{lem:Dirichlet-bound}
    Let $\tilde{\vec{P}}$ be a sample drawn from Dirichlet distribution $\Dir(nP_1+1,\dots,nP_L+1)$ for $n>0$ and  $\vec{P} = (P_1,\dots,P_L)\in\mathcal{P}_L$. For all $n>L$ and $\varepsilon>0$, we have
    \begin{align*}
        \prob{\KL{\vec{P}}{\tilde{\vec{P}}} \geq \varepsilon} \leq C'_1 n^L\exp(-n\varepsilon),
    \end{align*}
    where $C'_1=2^L C_1$ for $C_1$ defined in \eqref{eq:c_1-def}. 
\end{lem}

Using Pinsker's inequality, we can derive the concentration inequality for $l_1$-norm.
\begin{cor}\label{cor:Dirichlet-bound}
	For $\tilde{\vec{P}}$ and  $\vec{P}$ defined in Lemma~\ref{lem:Dirichlet-bound}, we have
	\begin{align*}
        \prob{\left|\vec{P} - \tilde{\vec{P}}\right| \geq \varepsilon} \leq C'_1 n^L\exp(-2n\varepsilon^2)
    \end{align*}
    for any $\varepsilon>0$. 
\end{cor}

\begin{proof}[Proof of Lemma~\ref{lem:Dirichlet-bound}]
Using \eqref{eq:starling-inequality}, we have
\begin{align*}
    &\prob{\KL{\vec{P}}{\tilde{\vec{P}}} \geq \varepsilon}\\
    &\leq \int \indi{\KL{\vec{P}}{\tilde{\vec{p}}})\geq \varepsilon} \frac{\Gamma(n+L)}{\prod_{i=1}^L \Gamma(nP_i+1)} \prod_{i=1}^L \tilde{p}_i^{(nP_i)}\intd{\tilde{\vec{p}}}\\
    &\leq \int \indi{\KL{\vec{P}}{\tilde{\vec{p}}})\geq \varepsilon} \frac{(n+L)^L\Gamma(n+1)}{\prod_{i=1}^L \Gamma(nP_i+1)} \prod_{i=1}^L \tilde{p}_i^{(nP_i)}\intd{\tilde{\vec{p}}}\\
    &\leq \int \indi{\KL{\vec{P}}{\tilde{\vec{p}}}\geq \varepsilon}2^Ln^L\\
    &\quad~~~~~\frac{\sqrt{2\pi}e^{1/6}(n+1)^{n+1/2}e^{-n-1}}{\prod_{i=1}^L \sqrt{2\pi}(nP_i+1)^{nP_i+1/2}e^{-nP_i-1}} \prod_{i=1}^L\tilde{p}_i^{(nP_i)} \intd{\tilde{\vec{p}}}\\
    &\leq C'_1n^L\int \indi{\KL{\vec{P}}{\tilde{\vec{p}}}\geq \varepsilon} \\
    &\quad~~~~~\frac{(n+1)^{n+1/2}}{\prod_{i=1}^L (n\hat{P}_i+1)^{n\hat{P}_i+1/2}}\prod_{i=1}^L (\tilde{P}_i)^{n\hat{P}_i}\intd{\tilde{\vec{p}}}\\
    &\leq C_1 n^L\int \indi{\KL{\vec{P}}{\tilde{\vec{p}}}\geq \varepsilon} \exp\left(F_n(\tilde{\vec{p}},\vec{P})\right) \intd{\tilde{\vec{p}}}\\
    &\leq C_1n^L\exp\left(-n\varepsilon\right),
\end{align*}
for $C_1$ and $F_n$ defined in \eqref{eq:c_1-def} and \eqref{eq:F-def}, respectively. 
\end{proof}



Lastly, we state two simple lemmas, which are useful for analysis. The first is about the characteristic of function $\mu$.

\begin{lem}\label{lem:characters-of-mu}
For $\vec{x},\vec{y},\vec{z}\in \mathcal{P}_L$ and function $\mu$ defined in \eqref{eq:mu-def}, we have
\begin{align*}
    |\mu(\vec{x},\vec{z}) - \mu(\vec{y},\vec{z})| \leq \frac12 |\vec{x}-\vec{y}|.
\end{align*}
Here, $|\vec{x}|$ is the $l_1$-norm of $\vec{x}$ defined as $\sum_{i=1}^L |x_i|$.
\end{lem}

We can confirm it by simple calculation. The second is used for bounding the confidence bound.

\begin{lem}\label{lem:setting-f}
    Let $f(C,\varepsilon,\delta)$ be 
    \begin{align*}
    	&f(C,\varepsilon,\delta) =\\
	&\quad \frac{1}{\varepsilon}\log \frac{C}{\delta} + \frac{1+\sqrt{5}}{2}\left(\frac{2L^2}{\varepsilon^2}\left(\log\left(\frac{1}{\varepsilon}\log \frac{C}{\delta} \right)\right)^2\right). 
    \end{align*}
    Then, for all $n \geq f(C,\varepsilon,\delta)$, we have
    \begin{align*}
    	Cn^L\exp(-n\varepsilon) \leq \delta.
    \end{align*}
\end{lem}

\begin{proof}
	Let $C_\delta$ be
\begin{align*}
	C_\delta =  \frac{1+\sqrt{5}}{2}\left(\frac{2L^2}{\varepsilon^2}\left(\log\left(\frac{1}{\varepsilon}\log \frac{C}{\delta} \right)\right)^2\right).
\end{align*}
If we set $n$ as
\begin{align*}
    	n &\geq \frac{1}{\varepsilon}\log \frac{C}{\delta} + C_\delta,
    \end{align*}
    we have
    \begin{align*}
    	&Cn^L\exp(-n\varepsilon)\\
	&\leq \delta \left(\frac{1}{\varepsilon}\log \frac{C}{\delta} + C_\delta\right)^L\exp(- \varepsilon C_\delta)\\
	&\leq \delta \exp\left(L\log\left(\frac{1}{\varepsilon}\log \frac{C}{\delta} + C_\delta\right) - \varepsilon C_\delta\right)\\
	&\leq \delta \exp\left(L\log\left(\frac{1}{\varepsilon}\log \frac{C}{\delta} \right)\log(1+ C_\delta) - \varepsilon C_\delta\right)\\
	&\leq \delta \exp\left(L\log\left(\frac{1}{\varepsilon}\log \frac{C}{\delta} \right)\sqrt{1+ C_\delta} - \varepsilon C_\delta\right)\\
	&\leq \delta.
    \end{align*}
\end{proof}

\section{Proof of Theorem~\ref{thm:thompson-condorcet-regret}} \label{sec:Proof-of-thompson-condorcet-regret}
We first introduce several events that is used in the proof. Let $E_N(t; \vec{n})$ be the event 
\begin{align*}
	E_N(t;  \vec{n}) = \left\{N_1(t) = n_1, \dots, N_K(t) = n_K\right\},
\end{align*}
where  $\vec{n}=(n_1,\dots,n_K)$ and $N_i(t)$ represents the number of times that arm $i$ is pulled before $t$-th round.  

%
%

We define two more events, $\arm{E}{i}_P(t), \arm{E}{i}_\theta(t)$ as 
\begin{align*}
    \arm{E}{i}_P(t) &= \left\{\KL{\hat{\vec{P}}^{(i)}(t)}{\vecarm{P}{i}} \leq \varepsilon' \right\}\\
    \arm{E}{i}_\theta(t) &= \left\{\KL{\hat{\vec{P}}^{(i)}(t)}{\vecarm{\theta}{i}} \leq \frac{1}{1+\varepsilon} \KL{\vecarm{P}{i}}{\vecarm{P^*}{i}}\right\}
\end{align*}
for some $\varepsilon,\varepsilon'>0 $ and the empirical distribution $\hat{\vec{P}}^{(i)}(t) = (C_1^{(i)}(t), \dots, C_L^{(i)}(t))^\top / N_i(t)$.

The probabilities of not having these events are bounded as follows.
\begin{lem}\label{lem:Event_lemma}
    \begin{align}
    	&\prob{\overline{\arm{E}{i}_P(t)} , N_i(t) = n} \leq C_1n^L\exp\left(-n\varepsilon'\right), \label{eq:E_P_bound}\\
	&\prob{\overline{\arm{E}{i}_\theta(t)} , N_i(t) = n} \notag\\
	&\quad ~~~~~~~~\leq C'_1n^L \exp\left(-\frac{n}{1+\varepsilon} \KL{\vecarm{P}{i}}{\vecarm{P^*}{i}}\right).\label{eq:E_theta_bound}
    \end{align}
\end{lem}

\begin{proof}
 Eq.~\eqref{eq:E_P_bound} follows directly from Lemma~\ref{lem:multinomial-bound}.  Eq.~\eqref{eq:E_theta_bound} can be derived as follows.
 \begin{align*}
 	&\prob{\overline{\arm{E}{i}_\theta(t)} , N_i(t) = n}\\
	& = \prob{\KL{\hat{\vec{P}}^{(i)}(t)}{\vecarm{\theta}{i}} \geq \frac{1}{1+\varepsilon} \KL{\vecarm{P}{i}}{\vecarm{P^*}{i}} , N_i(t) = n}\\
	& = \sum_{\vec{p}} \mathbb{P}\biggl[\KL{\vec{p}}{\vecarm{\theta}{i}} \geq \frac{\KL{\vecarm{P}{i}}{\vecarm{P^*}{i}}}{1+\varepsilon} , \\
	&~~~~~~~~~~ N_i(t) = n \biggm| \hat{\vec{P}}^{(i)}(t) = \vec{p}\biggr]	\prob{\hat{\vec{P}}^{(i)}(t) = \vec{p}}\\
	&\leq C'_1 n^L\exp\left(-\frac{n}{1+\varepsilon}\KL{\vecarm{P}{i}}{\vecarm{P^*}{i}}\right)\sum_{\vec{p}}  \prob{\hat{\vec{P}}^{(i)}(t) = \vec{p}}\\
	&= C'_1 n^L\exp\left(-\frac{n}{1+\varepsilon} \KL{\vecarm{P}{i}}{\vecarm{P^*}{i}}\right),
 \end{align*}
 where the inequality holds from Lemma~\ref{lem:Dirichlet-bound}.
\end{proof}

We introduce two probability vectors $\vecarm{x}{i},\vecarm{y}{i} \in \mathcal{P}_L$  that relates with events $\arm{E}{i}_P(t)$ and $\arm{E}{i}_\theta(t)$ as follows.
\begin{align*}
    \vecarm{x}{i} &= \argmax_{\vec{x} \in \mathcal{P}_L} \mu(\vec{x}, \vecarm{P}{a^*_\Condorcet}) \quad \mathrm{s.t} ~~\KL{\vec{x}}{\vecarm{P}{i}} \leq \varepsilon'\\
    \vecarm{y}{i} &= \argmax_{\vec{y}\in \mathcal{P}_L} \mu(\vec{y}, \vecarm{P}{a^*_\Condorcet})\\
    &\quad ~\mathrm{s.t.}~~\exists \vec{x}: \KL{\vec{x}}{\vecarm{P}{i}} \leq \varepsilon'\\
    &\quad ~~~~~~~~~~~~~ \KL{\vec{x}}{\vec{y}} \leq \frac{1}{1+\varepsilon} \KL{\vecarm{P}{i}}{\vecarm{P^*}{i}}.
\end{align*}
By definition, we have $\mu(\vecarm{P}{i},\vecarm{P}{a^*_\Condorcet}) \leq \mu(\vecarm{x}{i},\vecarm{P}{a^*_\Condorcet}) \leq \mu(\vecarm{y}{i},\vecarm{P}{a^*_\Condorcet})$. Now, we show that we have $\mu(\vecarm{y}{i},\vecarm{P}{a^*_\Condorcet})\leq 1/2$ by taking sufficiently small $\varepsilon'$.

\begin{lem}\label{lem:setting_x_y}
    Assume that $\arm{P}{i}_k > 0$ for all $k\in[L]$, and let $\arm{P}{i}_{\mathrm{min}}$ be $\arm{P}{i}_{\mathrm{min}} = \min_{k\in[L]} \arm{P}{i}_k$. If $\varepsilon'$ is set as $\varepsilon'< (\arm{P}{i}_{\mathrm{min}})^2 $ and satisfies
    \begin{align}
        \sqrt{\frac{\varepsilon'}2}\frac{1}{\arm{P}{i}_{\mathrm{min}}-\sqrt{\frac{\varepsilon'}2}} \leq \frac12 \frac{\frac{\varepsilon}{1+\varepsilon} \KL{\vecarm{P}{i}}{\vecarm{P^*}{i}}}{C_2 + \frac{\varepsilon}{1+\varepsilon} \KL{\vecarm{P}{i}}{\vecarm{P^*}{i}}} \label{eq:varepsilon_cond.}
    \end{align}
    for 
   \begin{align*}
    C_2 = -\sum_{k=1}^L \arm{P}{i}_k \log \arm{P}{i}_k - \sum_{k=1}^L 2\log \arm{P}{i}_k,
    \end{align*}
    then the inequality
    \begin{align*}
       \mu(\vecarm{y}{i},\vecarm{P}{a^*_\Condorcet}) < \frac12
    \end{align*}
    holds.
\end{lem}

\begin{proof}
By the definition of $\vecarm{x}{i}$ and $\vecarm{y}{i}$, there exists vector $\vec{x}'$ that satisfies 
\begin{align}
\KL{\vec{x}'}{\vecarm{P}{i}} &\leq \varepsilon' \label{eq:proof_x_y_temp} \\
\KL{\vec{x}'}{\vecarm{y}{i}}&\leq \frac{1}{1+\varepsilon} \KL{\vecarm{P}{i}}{\vecarm{P^*}{i}}. \label{eq:proof_x_y_temp2}
\end{align}
Using the convexity of the KL divergence, we have
\begin{align}
    &\KL{\vec{x}'}{\vecarm{y}{i}} \notag\\
    &> \KL{\vecarm{P}{i}}{\vecarm{y}{i}} + \sum_{k=1}^L (x'_k-\arm{P}{i}_k) \log\frac{\arm{P}{i}_k}{\arm{y}{i}_k}\notag\\
    &\geq \KL{\vecarm{P}{i}}{\vecarm{y}{i}} + \sum_{k=1}^L (x'_k-\arm{P}{i}_k) \log \arm{P}{i}_k\notag\\
    &\quad\quad - \sum_{k=1}^L (x'_k-\arm{P}{i}_k) \log\arm{y}{i}_k\notag\\
    &\geq \KL{\vecarm{P}{i}}{\vecarm{y}{i}} - \sqrt{\frac{\varepsilon'}2} \sum_{k=1}^L  \log \frac{1}{\arm{P}{i}_k}\notag \\
    &\quad\quad - \sqrt{\frac{\varepsilon'}2} \sum_{k=1}^L  \log \frac{1}{\arm{y}{i}_k}, \label{eq:proof_x_y_temp3}
\end{align}
where the last inequality holds from \eqref{eq:proof_x_y_temp} and Pinsker's inequality. By  \eqref{eq:proof_x_y_temp2}, we have
\begin{align*}
    &\sum_{k=1}^L x'_k \log\frac{x'_k}{\arm{y}{i}_k} \leq  \frac{\KL{\vecarm{P}{i}}{\vecarm{P^*}{i}}}{1+\varepsilon} \\
    \Leftrightarrow&\sum_{k=1}^L x'_k \log\frac{1}{\arm{y}{i}_k} \leq  \frac{\KL{\vecarm{P}{i}}{\vecarm{P^*}{i}}}{1+\varepsilon}  - \sum_{k=1}^L x'_k \log x'_k.
\end{align*}
Using \eqref{eq:proof_x_y_temp} and Pinsker's inequality, we have $x'_k \geq \arm{P}{i}_{\mathrm{min}} - \sqrt{\frac{\varepsilon'}2}$. Letting $\Delta P = \arm{P}{i}_{\mathrm{min}} - \sqrt{\frac{\varepsilon'}2}$, we have
\begin{align*}
   &\sum_{k=1}^L \log\frac{1}{\arm{y}{i}_k} \\
   &\quad \leq \frac{1}{\Delta P}\left( \frac{1}{1+\varepsilon} \KL{\vecarm{P}{i}}{\vecarm{P^*}{i}}- \sum_{k=1}^L x'_k \log x'_k\right)\\
   &\quad \leq \frac{1}{\Delta P}\biggl( \frac{1}{1+\varepsilon} \KL{\vecarm{P}{i}}{\vecarm{P^*}{i}}\\
   &\quad\quad  - \sum_{k=1}^L \arm{P}{i}_k \log \arm{P}{i}_k - \sum_{k=1}^L (x'_k - \arm{P}{i}_k) \log \arm{P}{i}_k\biggr)\\
   &\quad \leq \frac{1}{\Delta P}\biggl( \frac{1}{1+\varepsilon} \KL{\vecarm{P}{i}}{\vecarm{P^*}{i}}\\
   &\quad\quad\quad  - \sum_{k=1}^L \arm{P}{i}_k \log \arm{P}{i}_k - \sum_{k=1}^L \log \arm{P}{i}_k\biggr),
\end{align*}
where the second inequality holds for the convexity of $ x'_k \log x'_k$. Substituting  it with  \eqref{eq:proof_x_y_temp3} yields
\begin{align*}
    &\KL{\vec{x}'}{\vecarm{y}{i}} \\
    &\quad \geq \KL{\vecarm{P}{i}}{\vecarm{y}{i}} - \sqrt{\frac{\varepsilon'}2} \sum_{k=1}^L  \log \frac{1}{\arm{P}{i}_k} \\
    &\quad\quad  - \sqrt{\frac{\varepsilon'}2}\frac{1}{\Delta P}\biggl( \frac{1}{1+\varepsilon} \KL{\vecarm{P}{i}}{\vecarm{P^*}{i}}\\
   &\quad\quad\quad\quad\quad  + \sum_{k=1}^L \arm{P}{i}_k \log \arm{P}{i}_k + \sum_{k=1}^L \log \arm{P}{i}_k\biggr).
\end{align*}
Again, using \eqref{eq:proof_x_y_temp2}, we have
\begin{align*}
    &\KL{\vecarm{P}{i}}{\vecarm{y}{i}} \\
    &\quad \leq \frac{1+ \sqrt{\frac{\varepsilon'}2}\frac{1}{\Delta P}}{1+\varepsilon} \KL{\vecarm{P}{i}}{\vecarm{P^*}{i}} + \sqrt{\frac{\varepsilon'}2}\frac{1}{\Delta P}C_2.
\end{align*}
Therefore, if \eqref{eq:varepsilon_cond.} holds, we have 
\begin{align*}
    \KL{\vecarm{P}{i}}{\vecarm{y}{i}} < \KL{\vecarm{P}{i}}{\vecarm{P^*}{i}},
\end{align*}
which implies $\mu(\vecarm{y}{i},\vecarm{P}{a^*_\Condorcet}) \leq 1/2$.
\end{proof}

From now, we only consider the case of $\mu(\vecarm{y}{i},\vecarm{P}{a^*_\Condorcet}) < 1/2$, and we denote $\Delta'_i = 1/2 - \mu(\vecarm{y}{i},\vecarm{P}{a^*_\Condorcet})$.  We decompose the number of times to pull sub-optimal arm $i\neq a^*_\Condorcet$ as follows.

\begin{align}
	\expect{N_i(T)} &= \sum_{t=1}^T \prob{a_t = i}\notag\\
	 &= t_0 + \sum_{t=Kt_0}^T \prob{a_t = i, \arm{E}{i}_P(t), \arm{E}{i}_\theta(t)} \notag\\
	 &\quad  + \sum_{t=Kt_0}^T \prob{a_t = i, \arm{E}{i}_P(t), \overline{\arm{E}{i}_\theta(t)}}\notag\\
	 &\quad  +  \sum_{t=Kt_0}^T \prob{a_t = i, \overline{\arm{E}{i}_P(t)}}. \label{eq:thompson-decomposition}
\end{align}

Each term is bounded by Lemmas~\ref{lem:thompson_sub_bound-1}--\ref{lem:thompson_sub_bound-3}.

\begin{lem}\label{lem:thompson_sub_bound-1}
    \begin{align*}
    	&\sum_{t=Kt_0}^T \prob{a_t=i,\overline{\arm{E}{i}_P(t)}} \leq 1+\frac{C_1L!}{(\varepsilon'_i)^L}.
    \end{align*}
\end{lem}

\begin{proof}
	Let $\arm{\tau}{i}_k$ be the round that $k$-th pull of arm $i$ happens, then we have
\begin{align*}
    &\sum_{t=Kt_0}^T \prob{a_t=i,\overline{\arm{E}{i}_P(t)}}\\
    &= \expect{\sum_{k=\tau_0}^T\indi{\overline{\arm{E}{i}_P(\arm{\tau}{i}_k+1)}}\sum_{t=\arm{\tau}{i}_k+1}^{\arm{\tau}{i}_{k+1}}\indi{a_t=i}}\\
    &= \expect{\sum_{k=t_0}^T\indi{\overline{\arm{E}{i}_P(\arm{\tau}{i}_k+1)}}}\\
    &\leq 1+C_1\sum_{k=1}^T k^L\exp(-k\varepsilon'_i) \quad(\because \text{ Lemma~\ref{lem:Event_lemma} })\\
    &\leq 1+\frac{C_1L!}{(\varepsilon'_i)^L}.
\end{align*}
\end{proof}

\begin{lem}\label{lem:thompson_sub_bound-2}
    \begin{align*}
    	&\sum_{t=Kt_0}^T \prob{a_t = i, \arm{E}{i}_P(t), \overline{\arm{E}{i}_\theta(t)}}\\
	&\quad~~~~~~ \leq\frac{ 1+\varepsilon}{\KL{\vecarm{P}{i}}{\vecarm{P^*}{i}}}\log T + O((\log\log T)^2).
    \end{align*}
\end{lem}

\begin{proof}
Let $L_i$ be 
\begin{align*}
	L_i &= \frac{1}{a}\log T \\
	&\quad+  \frac{1+\sqrt{5}}{2}\frac{2L^2}{a^2}\left(\log\left(\frac{1}{a}\log T \right)\right)^2 
\end{align*}
for $a = \frac{1}{ 1+\varepsilon}\KL{\vecarm{P}{i}}{\vecarm{P^*}{i}}$.
From lemma~\ref{lem:Event_lemma},  we have
\begin{align*}
	\forall n \geq L_1,~~\prob{N_i(t) = n, \arm{E}{i}_\theta(t) } \leq \frac{C_1}{T}.
\end{align*}
Therefore, 
\begin{align*}
    &\sum_{t=Kt_0}^T \prob{a_t = i, \arm{E}{i}_P(t), \overline{\arm{E}{i}_\theta(t)}}\\
    &\leq \sum_{t=Kt_0}^T \prob{a_t=i,\overline{\arm{E}{i}_\theta(t)}}\\
    &= L_i + \expect{ \sum_{t=\arm{\tau}{i}_{L_i}}^T \indi{a_t=i,\overline{\arm{E}{i}_\theta(t)}}}\\
    &\leq L_i + \sum_{t=Kt_0}^T \frac{C_1}{T}\\
    &\leq \frac{ 1+\varepsilon}{\KL{\vecarm{P}{i}}{\vecarm{P^*}{i}}}\log T + O((\log\log T)^2).
\end{align*}
\end{proof}

\begin{lem}\label{lem:thompson_sub_bound-3}
If $t_0$ satisfies 
\begin{align}
t_0 &\geq f\left(C_1,\varepsilon', \frac12\right) +\max_{i\neq a^*_\Condorcet} f\left(C'_1, \frac{\KL{\vecarm{P}{i}}{\vecarm{P^*}{i}}}{1+\varepsilon}, \frac12\right) \label{eq:t_0-def-1}
\end{align}
and 
\begin{align}
t_0 &\geq f\left(C'_1,\Delta'_\mathrm{min}, \frac12\right) + \frac{2}{(\Delta'_\mathrm{min})^2} \log 2^{L+1}\label{eq:t_0-def-2}
\end{align}
for $f(C,\varepsilon,\delta)$ defined in Lemma~\ref{lem:setting-f}, then
we have
    \begin{align*}
    	&\sum_{t=Kt_0}^T \prob{a_t = i, \arm{E}{i}_P(t), \arm{E}{i}_\theta(t)}\\
	&\quad \leq  4^{K+1} \frac{C'_1L!}{2^L(\Delta'_\mathrm{min})^{2L}} + 4^{K+1+\frac{L}{2}}\frac2{(\Delta'_\mathrm{min})^2},
    \end{align*}
    where $\Delta'_\mathrm{min} = \min_{i\neq a^*_\Condorcet} \frac12 - \mu(\vecarm{y}{i},\vecarm{P}{a^*_\Condorcet})$. 
\end{lem}
\begin{proof}
We first define the following events.
\begin{align*}
	&\arm{E}{i}_c(t) = \left\{\forall i\neq j~~ \mu(\vecarm{\theta}{i},\vecarm{\theta}{j}) \geq \frac12\right\}\\
	&E'_N(t; \vec{n}, \{\vec{P}_k\}_k) = \\
	&\quad~~~~ E_N(t; \vec{n}) \cap \left\{\forall k\in[K] ~~ \arm{\hat{\vec{P}}}{k} = \vec{P}_k \right\}.
\end{align*}
Since the algorithm continues to sample $\{\vecarm{\theta}{i}\}_i$ until the Condorcet winner exists, we have
\begin{align}
	&\prob{a_t = i \bigmid E'_N(t; \vec{n}, \{\vec{P}_k\}_k)}\notag\\
	&\quad~~~ = \frac{\prob{\arm{E}{i}_c(t) \bigmid E'_N(t; \vec{n}, \{\vec{P}_k\}_k)}}{\sum_{j=1}^K \prob{\arm{E}{j}_c(t) \bigmid E'_N(t; \vec{n}, \{\vec{P}_k\}_k)}} \label{eq:arm-pull-prob}\\
	&\quad~~~ \geq \prob{\arm{E}{i}_c(t) \bigmid E'_N(t; \vec{n}, \{\vec{P}_k\}_k)}. \label{eq:arm-pull-prob-2}
\end{align}
We also consider the following two events:
\begin{align*}
	&\arm{E}{i}_D(t) = \left\{\mu(\vecarm{\theta}{a^*_\Condorcet},\vecarm{\theta}{i})  \leq  \frac12\right\}\\ 
	&E_D(t) = \left\{|\vecarm{\theta}{a^*_\Condorcet}-\vecarm{P}{a^*_\Condorcet}| \geq 2\Delta'_{\mathrm{min}}\right\}.
\end{align*}
Then, for $\{\vec{P}_k\}_k$ satisfying $\arm{E}{i}_P(t)$, we have
\begin{align*}
	 &\prob{a_t = i \bigmid  \arm{E}{i}_\theta(t), E'_N(t;\vec{n}, \{\vec{P}_k\}_k)}\\
	 &= \prob{a_t = i, \arm{E}{i}_D(t) \bigmid  \arm{E}{i}_\theta(t), E'_N(t;\vec{n}, \{\vec{P}_k\}_k)}\\
	 &\leq \prob{\arm{E}{i}_D(t) \bigmid  \arm{E}{i}_\theta(t), E'_N(t;\vec{n}, \{\vec{P}_k\}_k)}\\
	 &\leq \prob{E_D(t) \bigmid  \arm{E}{i}_\theta(t), E'_N(t;\vec{n}, \{\vec{P}_k\}_k)},
\end{align*}
where the last inequality holds from the definition of $\Delta'_{\mathrm{min}}$ and Lemma~\ref{lem:characters-of-mu}, and  the first equation holds for 
\begin{align*}
&\prob{a_t = i, \overline{\arm{E}{i}_D(t)} \bigmid  \arm{E}{i}_\theta(t), E'_N(t;\vec{n}, \{\vec{P}_k\}_k)}\\
& = \prob{\overline{\arm{E}{i}_D(t)} \bigmid  \arm{E}{i}_\theta(t), E'_N(t;\vec{n}, \{\vec{P}_k\}_k)}\\
&\quad~ \frac{\prob{\arm{E}{i}_c(t) \bigmid \overline{\arm{E}{i}_D(t)},E'_N(t; \vec{n}, \{\vec{P}_k\}_k)}}{\sum_{j=1}^K \prob{\arm{E}{j}_c(t) \bigmid \overline{\arm{E}{i}_D(t)}, E'_N(t; \vec{n}, \{\vec{P}_k\}_k)}}~(\because \eqref{eq:arm-pull-prob})\\
&= 0.
\end{align*}
Since event $E_D(t)$ only depends on $N_{a^*_\Condorcet}(t)$ and $\arm{\hat{\vec{P}}}{a^*_\Condorcet}(t)$,  we have
\begin{align*}
&\prob{a_t = i, \arm{E}{i}_\theta(t), \arm{E}{i}_P(t), N_{a^*_\Condorcet}(t) = n, \arm{\hat{\vec{P}}}{a^*_\Condorcet}(t) = \vec{P}}\\
&\leq \prob{E_D(t), N_{a^*_\Condorcet}(t) = n, \arm{\hat{\vec{P}}}{a^*_\Condorcet}(t) = \vec{P}}.
\end{align*}

Moreover, from \eqref{eq:arm-pull-prob-2}, we have
\begin{align*}
&\prob{a_t = a^*_\Condorcet \bigmid E'_N(t; \vec{n}, \{\vec{P}_k\}_k)}\\
&\geq \prob{\arm{E}{a^*_\Condorcet}_c(t) \bigmid E'_N(t; \vec{n}, \{\vec{P}_k\}_k)}\\
&\geq \prob{\overline{E_D(t)}, \arm{E}{1}_\theta(t), \dots, \arm{E}{K}_\theta(t) \bigmid E'_N(t; \vec{n}, \{\vec{P}_k\}_k)}
\end{align*}
for $\{\vec{P}_k\}_k$ satisfies $\arm{E}{1}_P(t), \dots, \arm{E}{K}_P(t)$. Marginalizing the above for $\{\vec{P}_k\}_{k\neq a^*_\Condorcet}$ yields
\begin{align*}
&\prob{a_t = a^*_\Condorcet, N_{a^*_\Condorcet}(t) = n_{a^*_\Condorcet}, \arm{\hat{\vec{P}}}{a^*_\Condorcet}(t) = \vec{P}_{a^*_\Condorcet}}\\
&\geq \mathbb{P}\left[\overline{E_D(t)}, \forall j\neq a^*_\Condorcet~~(\arm{E}{j}_\theta(t) \cap \arm{E}{j}_P(t)), E_N(t;\vec{n}),\right.\\
&\quad~~~~~~~~~~ \left. N_{a^*_\Condorcet}(t) = n_{a^*_\Condorcet}, \arm{\hat{\vec{P}}}{a^*_\Condorcet}(t) = \vec{P}_{a^*_\Condorcet}\right]\\
&\geq \frac{1}{4^K} \prob{\overline{E_D(t)}, N_{a^*_\Condorcet}(t) = n_{a^*_\Condorcet}, \arm{\hat{\vec{P}}}{a^*_\Condorcet}(t) = \vec{P}_{a^*_\Condorcet}},
\end{align*}
where the last inequality holds from \eqref{eq:t_0-def-1} and lemma~\ref{lem:Event_lemma}.

Thus, the ratio of the probability of pulling arm $i$ to that of pulling arm $a^*_\Condorcet$ is bounded as 
\begin{align}
&\frac{\prob{a_t = i, \arm{E}{i}_\theta(t), \arm{E}{i}_P(t), N_{a^*_\Condorcet}(t) = n}}{\prob{a_t = a^*_\Condorcet, N_{a^*_\Condorcet}(t) = n}\notag}\\
&\leq 4^K \frac{\sum_{\vec{P}} \prob{E_D(t), N_{a^*_\Condorcet}(t) = n, \arm{\hat{\vec{P}}}{a^*_\Condorcet}(t) = \vec{P}}}{\sum_{\vec{P}} \prob{\overline{E_D(t)}, N_{a^*_\Condorcet}(t) = n, \arm{\hat{\vec{P}}}{a^*_\Condorcet}(t) = \vec{P}}}\label{eq:thompson_sub_bound-3-tmp1}
\end{align}
Now, for set $A$ defined as
\begin{align*}
A = \{\vec{P} \in \mathcal{P}_L \mid  |\vec{P}- \vecarm{P}{a^*_\Condorcet}| \leq \Delta'_\mathrm{min}\},
\end{align*}
we derive the upper-bound of the numerator in  \eqref{eq:thompson_sub_bound-3-tmp1} as follows.
\begin{align*}
&\sum_{\vec{P}\in A} \prob{E_D(t), N_{a^*_\Condorcet}(t) = n, \arm{\hat{\vec{P}}}{a^*_\Condorcet}(t) = \vec{P}}\\
&\quad\leq\sum_{\vec{P}\in A} \mathbb{P}\left[|\vecarm{\theta}{a^*_\Condorcet} - \vec{P}| \geq \Delta'_\mathrm{min},\right. \\
&\quad\quad~~~~~~~~~~~~~~\left.N_{a^*_\Condorcet}(t) = n, \arm{\hat{\vec{P}}}{a^*_\Condorcet}(t) = \vec{P}\right]\\
&\quad\leq C'_1n^L\exp(-2n(\Delta'_\mathrm{min})^2) \quad(\because \text{Lemma}~\ref{lem:Dirichlet-bound}),\\
&\sum_{\vec{P}\notin A} \prob{E_D(t), N_{a^*_\Condorcet}(t) = n, \arm{\hat{\vec{P}}}{a^*_\Condorcet}(t) = \vec{P}}\\
&\quad\leq \prob{N_{a^*_\Condorcet}(t) = n, |\vec{P}- \vecarm{P}{a^*_\Condorcet}| \geq \Delta'_\mathrm{min}}\\
&\quad\leq 2^L\exp\left(-\frac12 n (\Delta'_\mathrm{min})^2\right)\quad(\because \text{Lemma}~\ref{lem:multinomial-l1-bound}).
\end{align*}
Thus, we have
\begin{align*}
&\sum_{\vec{P}} \prob{E_D(t), N_{a^*_\Condorcet}(t) = n, \arm{\hat{\vec{P}}}{a^*_\Condorcet}(t) = \vec{P}}\\
&\leq C'_1n^L\exp(-2n(\Delta'_\mathrm{min})^2) + 2^L\exp\left(-\frac12 n (\Delta'_\mathrm{min})^2\right).
\end{align*}

For the denominator, we have
\begin{align}
&\sum_{\vec{P}} \prob{\overline{E_D(t)}, N_{a^*_\Condorcet}(t) = n, \arm{\hat{\vec{P}}}{a^*_\Condorcet}(t) = \vec{P}}\notag\\
&\geq \sum_{\vec{P}\in A} \prob{\overline{E_D(t)}, N_{a^*_\Condorcet}(t) = n, \arm{\hat{\vec{P}}}{a^*_\Condorcet}(t) = \vec{P}}\notag\\
&\geq \sum_{\vec{P}\in A} \mathbb{P}\left[|\vecarm{\theta}{a^*_\Condorcet}- \vec{P}| \leq \Delta'_\mathrm{min},\right. \notag\\
&\quad~~~~~~~~~~~~~~ \left.N_{a^*_\Condorcet}(t) = n, \arm{\hat{\vec{P}}}{a^*_\Condorcet}(t) = \vec{P}\right]\\
&\geq \left(1-C'_1n^L\exp(-2n(\Delta'_\mathrm{min})^2)\right)\notag\\
&\quad~~~~~~~~~~~~~~~ \left(1-2^L\exp\left(\frac{n(\Delta'_\mathrm{min})^2}{2}\right)\right),\label{eq:thompson_sub_bound_3_tmp1}\\
&\geq 1 -C'_1n^L\exp(-2n(\Delta'_\mathrm{min})^2) - 2^L\exp\left(\frac{n(\Delta'_\mathrm{min})^2}{2}\right). \notag
\end{align}
where \eqref{eq:thompson_sub_bound_3_tmp1} holds from Lemmas~\ref{lem:multinomial-l1-bound} and~\ref{lem:Dirichlet-bound}.
Therefore, since the denominator is smaller than $1/4$ from \eqref{eq:t_0-def-2}, we have
\begin{align*}
&\sum_{t=1}^T \prob{a_t = i, \arm{E}{P}_i(t), \arm{E}{\theta}_i(t) }\\
&= \sum_{t=1}^T \sum_{n=1}^{t} \prob{a_t = i, \arm{E}{P}_i(t), \arm{E}{\theta}_i(t), N_{a^*_\Condorcet}(t) = n}\\
&\leq \sum_{t=1}^T \sum_{n=1}^{t}  \frac{ 4^KC'_1n^L\exp(-2n(\Delta'_\mathrm{min})^2) + 2^L\exp\left(\frac{n(\Delta'_\mathrm{min})^2}{2}\right)}{1 -C'_1n^L\exp(-2n(\Delta'_\mathrm{min})^2) - 2^L\exp\left(\frac{n(\Delta'_\mathrm{min})^2}{2}\right)}\\
&\quad ~~~~~~~~~~~~~~\prob{a_t = a^*_\Condorcet, N_{a^*_\Condorcet}(t) = n}\\
&\leq 4^{K+1} \sum_{n=1}^T  C'_1n^L\exp(-2n(\Delta'_\mathrm{min})^2) \\
&\quad~~~+ 4^{K+1} \sum_{n=1}^T2^L\exp\left(\frac{n(\Delta'_\mathrm{min})^2}{2}\right)\\
&\leq 4^{K+1} \frac{C'_1L!}{2^L(\Delta'_\mathrm{min})^{2L}} + 4^{K+1+\frac{L}{2}}\frac2{(\Delta'_\mathrm{min})^2}.
\end{align*}

\end{proof}

Combining Lemmas~\ref{lem:thompson_sub_bound-1}--\ref{lem:thompson_sub_bound-3} yields the proof of Theorem~\ref{thm:thompson-condorcet-regret}.

\begin{proof}[Proof of Theorem~\ref{thm:thompson-condorcet-regret}]
	Since regret $R^\Condorcet_T$ is defined as 
	\begin{align*}
		R^\Condorcet_T = \sum_{i\neq a^*_\Condorcet} \Delta^\Condorcet_i N_i(T),
	\end{align*}
	it suffices to bound the expectation $\expect{N_i(T)}$, which is decomposed as
	\begin{align*}
		\expect{N_i(T)} &= \sum_{t=1}^T \prob{a_t = i}\\
		 &\leq t_0 + \sum_{t=K\tau_0}^T \prob{a_t = i, \arm{E}{i}_P(t), \arm{E}{i}_\theta(t)} \\
		 &\quad  + \sum_{t=K\tau_0}^T \prob{a_t = i, \arm{E}{i}_P(t), \overline{\arm{E}{i}_\theta(t)}}\\
		 &\quad  +  \sum_{t=K\tau_0}^T \prob{a_t = i, \overline{\arm{E}{i}_P(t)}}.
	\end{align*}
	From Lemmas~\ref{lem:thompson_sub_bound-1}--\ref{lem:thompson_sub_bound-3}, we have
	\begin{align*}
		&\expect{N_i(T)} \\
		 &\leq \frac{ 1+\varepsilon}{\KL{\vecarm{P}{i}}{\vecarm{P^*}{i}}}\log T +  C_1  \\
		  &\quad + t_0+ 1+\frac{C_1L!}{(\varepsilon'_i)^L} +\frac{ 4^{K+1}C'_1L!}{2^L(\Delta'_\mathrm{min})^{2L}} + \frac{4^{K+1+\frac{L+1}{2}}}{(\Delta'_\mathrm{min})^2}\\
		  &\quad +  \frac{1+\sqrt{5}}{2}\frac{2L^2}{\left(\frac{\KL{\vecarm{P}{i}}{\vecarm{P^*}{i}}}{ 1+\varepsilon}\right)^2}\\
		  &\quad~~~~~~~~~~ \left(\log\left(\frac{\KL{\vecarm{P}{i}}{\vecarm{P^*}{i}}}{1+\varepsilon}\log T \right)\right)^2
	\end{align*}
	for $t_0$ satisfying \eqref{eq:t_0-def-1} and \eqref{eq:t_0-def-2}. Hence, we have

\begin{align}
		&\expect{R^\Condorcet_T} \notag\\
		&\leq \sum_{i\neq a^*_\Condorcet} \Delta^\Condorcet_i  \left(\frac{ 1+\varepsilon}{\KL{\vecarm{P}{i}}{\vecarm{P^*}{i}}}\log T \right. \notag\\
		  & \quad+C_1 + t_0+1+\frac{C_1L!}{(\varepsilon'_i)^L}+ \frac{ 4^{K+1}C'_1L!}{2^L(\Delta'_\mathrm{min})^{2L}} + \frac{4^{K+1+\frac{L+1}{2}}}{(\Delta'_\mathrm{min})^2}\notag\\
		  & \quad+ \frac{1+\sqrt{5} L^2}{\left(\frac{\KL{\vecarm{P}{i}}{\vecarm{P^*}{i}}}{ 1+\varepsilon}\right)^2}\left.\left(\log\left(\frac{\KL{\vecarm{P}{i}}{\vecarm{P^*}{i}}}{1+\varepsilon}\log T \right)\right)^2 \right).\label{eq:thompson-precise-regret}
\end{align}

\end{proof}

Lastly, we prove Lemma~\ref{lem:KL-bigger-mu}, which states the proved regret bound can be arbitrarily smaller than the regret lower bound in the standard dueling bandit problem.

\begin{proof}[Proof of Lemma~\ref{lem:KL-bigger-mu}]
 Let $\vecarm{P}{1}$ and $\vecarm{P}{a^*_\Condorcet}$ be
\begin{align*}
&\vecarm{P}{1} = \left(\varepsilon, 1-\varepsilon' \right)\\
&\vecarm{P}{a^*_\Condorcet} = \left(\mathrm{e}^{\frac{-1}{\varepsilon'}}, 1-\mathrm{e}^{\frac{-1}{\varepsilon'}} \right)
\end{align*}
for $0\leq\varepsilon'\leq 1/2$. Then, we have
\begin{align*}
 &\mu(\vecarm{P}{a^*_\Condorcet}, \vecarm{P}{1})\\
 & \leq \frac12 + \frac12\varepsilon' - \frac12\mathrm{e}^{\frac{-1}{\varepsilon'}} + \varepsilon'\mathrm{e}^{\frac{-1}{\varepsilon'}}\\
 &\leq \frac12 + \frac12\varepsilon'.
\end{align*}
Since $d(x,1/2)$ monotonically increases in $x\geq 1/2$, we have
\begin{align*}
&d(\mu(\vecarm{P}{a^*_\Condorcet}, \vecarm{P}{1}),1/2)\\
&= \left(\frac12 + \frac12\varepsilon' \right)\log (1 + \varepsilon' )\\ 
&\quad +\left(\frac12 - \frac12\varepsilon'\right) \log (1 - \varepsilon' )\\
&\leq (\varepsilon')^2
\end{align*}
where the last inequality holds from $x \geq \log(1+x)$. 
Now, since $\vecarm{P^*}{1} = \vecarm{P}{a^*_\Condorcet}$,  we have
\begin{align*}
	&\KL{\vecarm{P}{1}}{\vecarm{P^*}{1}}\\
	&= \varepsilon' \log \frac{\varepsilon' }{\mathrm{e}^{\frac{-1}{\varepsilon'}}} + (1-\varepsilon') \log \frac{1-\varepsilon' }{1-\mathrm{e}^{\frac{-1}{\varepsilon'}}}\\
	&= 1 + \varepsilon' \log \varepsilon' + (1-\varepsilon') \log (1-\varepsilon') \\
	& \quad - (1-\varepsilon') \log (1-\mathrm{e}^{\frac{-1}{\varepsilon'}})\\
	&\geq 1-\log 2.
\end{align*}
Thus, if we set $\varepsilon' \leq  \sqrt{(1-\log 2)\varepsilon}$, we have
\begin{align*}
 \frac{d(\mu(\vecarm{P}{a^*_\Condorcet}, \vecarm{P}{1}),1/2)}{\KL{\vecarm{P}{1}}{\vecarm{P^*}{1}}} \leq \frac{(\varepsilon')^2}{1-\log 2} \leq \varepsilon.
\end{align*}
\end{proof}

\section{Proof of Theorem~\ref{thm:thompson-fails-borda}} \label{sec:Proof-of-thompson-fails-borda}
In this section, we prove the Theorem~\ref{thm:thompson-fails-borda}. The proof is inspired by \citet{Honda2014}.

\begin{proof}[Proof of Theorem~\ref{thm:thompson-fails-borda}]
Consider the case of $K=3, L=5$, where the distributions of arms are
\begin{align}
&\vecarm{P}{1} = \left(\varepsilon,\varepsilon,1-4\varepsilon,\varepsilon,\varepsilon\right)^\top,\notag\\
&\vecarm{P}{2} = \left(\varepsilon,\frac14+\varepsilon,\frac12-4\varepsilon,\frac14+\varepsilon,\varepsilon\right)^\top,\notag\\
&\vecarm{P}{3} = \left(\frac12,\frac14, \varepsilon, \frac14-2\varepsilon,\varepsilon\right)^\top \label{eq:thompson-fails-distributions}
\end{align}
for $\frac1{8}>\varepsilon>0$. By simple calculation, we have
\begin{align*}
	&\mu(\vecarm{P}{1},\vecarm{P}{2}) = \frac12,\\
	&\mu(\vecarm{P}{1},\vecarm{P}{3}) = \frac34 +\frac{\varepsilon}{4} -\frac{5\varepsilon^2}2,\\
	&\mu(\vecarm{P}{2},\vecarm{P}{3}) = \frac34 -\frac{5\varepsilon^2}2,
\end{align*}
and thus arm $1$ is the Borda winner. Now, let $\vec{P}'$ be
\begin{align*}
	\vec{P}' = \left(\frac12,\frac14-2\varepsilon, \varepsilon, \frac14,\varepsilon\right)^\top,
\end{align*}
and consider the event that the agent pulls $n$ samples and gets $\vecarm{C}{3} = n\vec{P}'$, where the $k$-th element of $\vecarm{C}{3}$ represents the number of times that the agent receives feedback $k\in[L]$ from pulling arm 3.
By the discussion similar to the proof of  Lemma~\ref{lem:multinomial-bound}, the probability of this event is bounded as follows.
\begin{align*}
	&\prob{\vecarm{C}{3} = n\vec{P}'}\\
    &=\frac{\Gamma(n+1)}{\prod_{i=1}^L \Gamma(nP'_i+1)}\prod_{i=1}^L (\arm{P}{3}_i)^{nP'_i}\\ 
    &\geq \frac{\sqrt{2\pi}(n+1)^{n+1/2}\mathrm{e}^{-n-1}}{\prod_{i=1}^L \sqrt{2\pi}\mathrm{e}^{1/6}(nP'_i+1)^{nP'_i+1/2}\mathrm{e}^{-nP'_i-1}}\prod_{i=1}^L (\arm{P}{3}_i)^{nP'_i}\\
    &= \frac{C_1}{\mathrm{e}^{\frac13}}\frac{(n+1)^{n+1/2}}{\prod_{i=1}^L (nP'_i+1)^{nP'_i+1/2}}\prod_{i=1}^L (\arm{P}{3}_i)^{nP'_i}\\
    &= \frac{C_1}{\mathrm{e}^{\frac13}}\exp\left(F(\vecarm{P}{3},\vec{P}')\right),
\end{align*}
where $F$ is defined in \eqref{eq:F-def}.
Since $F$ can be bounded as 
\begin{align*}
&F_n(\vecarm{P}{3},\vec{P}')\\
&= -n \sum_{i=1}^L P'_i\left(\log \frac{nP'_i+1}{n+1} - \log \arm{P}{3}_i\right) \\
    &\quad~~~~~~ + \frac12\left(\log(n+1) - \sum_{i=1}^L\log(nP'_i+1)\right)\\
&\geq -n\KL{\vec{P}'}{\vecarm{P}{3}} + \sum_{i=1}^L P'_i \log\left( \frac{nP'_i+P'_i}{nP'_i+1}\right)\\
&\quad\quad - \frac{L-1}2\log(n+1)\\
&\geq -n\KL{\vec{P}'}{\vecarm{P}{3}} + \sum_{i=1}^L P'_i \log P'_i - \frac{L-1}2\log(n+1),
\end{align*}
we have
\begin{align}
&\prob{\vecarm{C}{3}= n\vec{P}'} \notag\\
 &\quad  \geq C_{\vec{P}'}(n+1)^{-\frac{L-1}{2}} \exp(-n\KL{\vec{P}'}{\vecarm{P}{3}})\label{eq:proof_of_thompson_fails_tmp1}
\end{align}
for
\begin{align*}
	C_{\vec{P}'} = \frac{C_1}{\mathrm{e}^{1/3}}\prod_{k=1}^L (P'_k)^{P'_k}.
\end{align*}  
Note that this KL divergence can be bounded as
\begin{align}
\KL{\vec{P}'}{\vecarm{P}{3}} &= 2\varepsilon\left(\log\frac14 - \log\left(\frac14-2\varepsilon\right)\right)\notag\\
&\leq \frac{4\varepsilon^2}{\frac14-2\varepsilon}, \label{eq:P'-P3-KL}
\end{align}
where we use $\log (x+y) - \log x \leq y/x$ for the last inequality.

Now, we show that the agent is likely to pull arm $2$ in the case of $\vecarm{C}{3} = n\vec{P}'$. In this case, the estimated Borda score is calculated as
\begin{align*}
	\hat{B}_1 &= \frac12(\mu(\vecarm{P}{1},\vecarm{P}{2}) + \mu(\vecarm{P}{1},\vecarm{P}{3}))\\
			  &= \frac58 -\frac{7}{8}\varepsilon +\frac{5\varepsilon^2}4,\\
	\hat{B}_2 &= \frac12(\mu(\vecarm{P}{2},\vecarm{P}{1}) + \mu(\vecarm{P}{2},\vecarm{P}{3}))\\
			  &= \frac58 -\frac{3}{4}\varepsilon +\frac{5\varepsilon^2}4,\\
	\hat{B}_3 &= \frac12(\mu(\vecarm{P}{3},\vecarm{P}{1}) + \mu(\vecarm{P}{3},\vecarm{P}{2}))\\
			  &= \frac14 +\frac{13}{8}\varepsilon -\frac{5\varepsilon^2}2		  
\end{align*}
and thus arm $2$ has the largest estimated Borda score. By Lemma~\ref{lem:characters-of-mu}, the arm $2$ is pulled if 
\begin{align*}
|\vecarm{\theta}{3} - \vec{P}'| \leq \hat{B}_2 - \hat{B}_1  = \frac{\varepsilon}{8}.
\end{align*}
Therefore, if we denote $\vecarm{C}{3}(t)$ as vector $\vecarm{C}{3}$ at $t$-th round, the probability of pulling arm $2$ is bounded as
\begin{align}
	&\prob{a_t = 2, N_3(t) = n}\notag\\
	&\geq \prob{a_t = 2, \vecarm{C}{3}(t) = n\vec{P}'}\notag\\
	&\geq \prob{a_t = 2 \bigmid \vecarm{C}{3}(t) = n\vec{P}'}\prob{\vecarm{C}{3}(t) = n\vec{P}'}\notag\\
	&\geq \prob{|\vecarm{\theta}{3} - \arm{\hat{\vec{P}}}{3}(t)| \leq \frac{\varepsilon}{8}\bigmid \arm{\hat{\vec{P}}}{3}(t) = \vec{P}'}\notag\\
	&\quad ~~~~\prob{\vecarm{C}{3}(t) = n\vec{P}'}\notag\\
	&\geq \left( 1- C'_1n^L\exp\left(-\frac{n\varepsilon^2}{32}\right)\right)\notag\\
	&\quad~~~~~ C_{\vec{P}'} (n+1)^{-\frac{L-1}{2}}\exp\left(-\frac{4n\varepsilon^2}{\frac14-2\varepsilon}\right), \label{eq:proof_of_thompson_fails_tmp2}
\end{align}
where the last inequality holds from Corollary~\ref{cor:Dirichlet-bound}, \eqref{eq:proof_of_thompson_fails_tmp1}, and \eqref{eq:P'-P3-KL}.
In order to stop arm 2 to be estimated as the Borda winner, the agent must pull arm $3$ and update the estimation, though this rarely happens as formally discussed in the following. Since $|\vecarm{\theta}{3} - \vecarm{P}{3}| \geq \hat{B}_2-\hat{B}_3$ is necessary to pull arm $3$, the probability of pulling arm $3$ is bounded as
\begin{align}
	&\prob{a_t = 3, \vecarm{C}{3} = n\vec{P}'}\notag\\
	&\leq \prob{|\vecarm{\theta}{3} - \vec{P}'| \geq \frac12 - \frac{19}{8}\varepsilon + \frac{15}{4}\varepsilon^2 \bigmid \vecarm{C}{3} = n\vec{P}'}\notag\\
	&\leq \prob{|\vecarm{\theta}{3} - \vecarm{P}{3}| \geq \frac{13}{64}\bigmid \vecarm{C}{3} = n\vec{P}'} ~\left(\because 0\leq \varepsilon \leq \frac18\right)\notag\\
	&\leq C'_1n^L\exp\left(-\frac{169n}{2048}\right).\label{eq:proof_thompson_fails_tmp3}
\end{align}
For
\begin{align*}
n_{\varepsilon,T} = \frac{\frac14-2\varepsilon}{8\varepsilon^2}\log T,
\end{align*}
we can decompose the regret as
\begin{align}
	&\expect{R^\Borda_T}\notag\\
	&\geq \Delta^\Borda_2\expect{\sum_{t=1}^T \indi{a_t=2}}\notag\\
	&\geq \Delta^\Borda_2\min\left(\expect{\sum_{t=1}^T \indi{a_t=2}\bigmid N_3(T/2) < n_{\varepsilon,T}},\right.\notag\\
	&\quad~~~~~~~~~~~~~~ \left.\expect{\sum_{t=1}^T \indi{a_t=2}\bigmid N_3(T/2) \geq n_{\varepsilon,T}}\right).\label{eq:proof_thompson_fails_tmp4}
\end{align}
If we set 
\begin{align*}
t_0 \geq f\left(C'_1,\frac{32}{\varepsilon},\frac12\right)
\end{align*}
for $f$ defined in Lemma~\ref{lem:setting-f} and pull all arms $t_0$ times at the first $Kt_0$ rounds, we have
\begin{align*}
&\expect{\sum_{t=1}^T \indi{a_t=2} \bigmid N_3(T/2) < n_{\varepsilon,T}}\\
&\geq \sum_{t=Kt_0}^{T/2} \prob{a_t=2 \bigmid N_3(T/2) < n_{\varepsilon,T}}\\
&\geq \sum_{t=Kt_0}^{T/2} \left( 1- C'_1t_0^L\exp\left(-\frac{t_0\varepsilon^2}{32}\right)\right)\\
&\quad\quad \quad C_2n_{\varepsilon,T}^{-\frac{L-1}{2}}\exp\left(-\frac{4n_{\varepsilon,T}\varepsilon^2}{\frac14-2\varepsilon}\right)\\
&\geq \frac{C_2}2 \left(\frac{\frac14-2\varepsilon}{8\varepsilon^2} \log T\right)^{-\frac{L-1}{2}}\frac{\frac{T}{2}-Kt_0}{\sqrt{T}}\\
&=T^{1/2-o(1)}
\end{align*}
from the fact $t_0 \leq N_3(t) \leq n_{\varepsilon,T}$ and \eqref{eq:proof_of_thompson_fails_tmp2}.
Moreover, we have
\begin{align*}
&\expect{ \sum_{t=1}^T \indi{a_t=2} \bigmid N_3(T/2) \geq n_{\varepsilon,T}}\\
&\geq \sum_{t=Kt_0}^T \prob{ a_t=2 \bigmid  N_3(T/2) \geq n_{\varepsilon,T}}\\
&\geq \sum_{t=T/2}^T \prob{ a_t=2,\vecarm{C}{3}(t) = n_{\varepsilon,T}\vec{P}' \bigmid N_3(T/2) \geq n_{\varepsilon,T}}\\
&\geq \sum_{t=T/2}^T \prob{a_t=2\bigmid \vecarm{C}{3}(t) = n_{\varepsilon,T}\vec{P}',  N_3(T/2) \geq n_{\varepsilon,T}}\\
&\quad \quad \quad \quad \quad \prob{\vecarm{C}{3}(t) = n_{\varepsilon,T}\vec{P}'\bigmid N_3(T/2) \geq n_{\varepsilon,T}}.
\end{align*}
Since $\sum_{i=0}^L \arm{C}{3}_i (t)= N_3(t)$, the event 
\begin{align*}
\{\vecarm{C}{3}(t) = n_{\varepsilon,T}\vec{P}', t>T/2, N_3(T/2) \geq n_{\varepsilon,T}\}
\end{align*}
only occurs if and only if when we have
\begin{align*}
\{\vecarm{C}{3}(T/2) = n_{\varepsilon,T}\vec{P}', N_3(T/2) = N_3(t) = n_{\varepsilon,T}\}.
\end{align*}
Therefore, from \eqref{eq:proof_of_thompson_fails_tmp2} and \eqref{eq:proof_thompson_fails_tmp3}, the probability of this event can be bounded as
\begin{align*}
&\expect{ \sum_{t=Kt_0}^T \indi{a_t=2}\bigmid N_3(T/2) \geq n_{\varepsilon,T}}\\
&\geq \sum_{t=T/2}^T \left( 1- C_1t_0^L\exp\left(-\frac{t_0\varepsilon^2}{32}\right)\right)C_2n_{\varepsilon,T}^{-\frac{L-1}{2}}\\
&\quad \quad  \exp\left(-\frac{4n_{\varepsilon,T}\varepsilon^2}{\frac14-2\varepsilon}\right)\left(1-C_1\exp\left(-\frac{169n_{\varepsilon,T}}{2048}\right)\right)^{t-T/2+1}\\
&\geq \frac{C_2}2 \left(\frac{\frac14-2\varepsilon}{8\varepsilon^2}\log T \right)^{-\frac{L-1}{2}}\frac{1}{\sqrt{T}}\\
&\quad \quad \quad \quad \quad \sum_{t'=1}^{T/2} (1-C_1T^{-\frac{169}{2048}\frac{\frac14-2\varepsilon}{8\varepsilon^2}})^{t'}\\
&= \frac{C_2}2 \left(\frac{\frac14-2\varepsilon}{8\varepsilon^2}\log T \right)^{-\frac{L-1}{2}}\frac{1}{C_1\sqrt{T}}T^{\frac{169}{2048}\frac{\frac14-2\varepsilon}{8\varepsilon^2}}\\
&\quad ~~~\left((1-C_1T^{-\frac{169}{2048}\frac{\frac14-2\varepsilon}{8\varepsilon^2}})-(1-C_1T^{-\frac{169}{2048}\frac{\frac14-2\varepsilon}{8\varepsilon^2}})^{T/2+1} \right).
\end{align*}
Therefore, if we set $\varepsilon \leq \frac{13}{256}$, we have
\begin{align*}
\expect{ \sum_{t=Kt_0}^T \indi{a_t=2}\bigmid N_3(T/2) \geq n_{\varepsilon,T}} \geq T^{\frac{3}{32}-o(1)}.
\end{align*}
Therefore, by \eqref{eq:proof_thompson_fails_tmp4}, we have
\begin{align*}
	\expect{R^\Borda_T} \geq T^{\frac{3}{32}-o(1)}
\end{align*}
if $\varepsilon \leq \frac{13}{256}$ and $t_0 \geq  f(C'_1,32/\varepsilon, 12)$.
\end{proof}

\section{Proof of Theorem~\ref{thm:UCB-borda-regret}} \label{sec:Proof-of-Borda-regret}
First, we introduce the following lemma.
\begin{lem}\label{lem:w-bound}
    Let $\vecarm{w}{i},\arm{\hat{\vec{w}}}{i}(t)$ be  $L$-dimensional vectors, the $k$-th elements of which are 
    \begin{align*}
    	&\arm{w}{i}_k = \left(\sum_{l=1}^k \arm{P}{i}_l - \frac12 \arm{P}{i}_k\right),\\
		&\arm{\hat{w}}{i}_k(t) = \left(\sum_{l=1}^k \arm{\hat{P}}{i}_l(t) - \frac12 \arm{\hat{P}}{i}_k(t)\right),
    \end{align*}
    respectively.
    Then, we have
    \begin{align*}
        \prob{\max_{k\in[L]} |\arm{w}{i}_k- \arm{\hat{w}}{i}_k| \geq \varepsilon} \leq 2\exp(-2n\varepsilon^2).
    \end{align*}
\end{lem}

\begin{proof}
    Let $\arm{F}{i}_k,\arm{\hat{F}}{i}_k$ be $\arm{F}{i}_k = \sum_{l=1}^k \arm{P}{i}_l,\,\arm{\hat{F}}{i}_k(t) = \sum_{l=1}^k \arm{\hat{P}}{i}_l(t)$ respectively. Then, the inequality
    \begin{align*}
        &|w^i_k-\hat{w}^i_k(t)|\\
         &\quad= \left| \frac{F^i_k-\hat{F}^i_k(t) + F^i_{k-1}-\hat{F}^i_{k-1}(t)}{2}\right| \\
        &\quad= \frac12\left|F^i_k-\hat{F}^i_k(t)\right| + \frac12\left|F^i_{k-1}-\hat{F}^i_{k-1}(t)\right| \\
        &\quad\leq \max(|F^i_k-\hat{F}^i_k(t)|, |F^i_{k-1}-\hat{F}^i_{k-1}(t)|) 
    \end{align*}
    holds for all $k\in[L]$. Therefore, we have
    \begin{align*}
    	 &\prob{\max_{k \in [L]} |\arm{w}{i}_k- \arm{\hat{w}}{i}_k(t)| \geq \varepsilon} \\
	 &\leq \prob{\max_{k \in [L]} \max(|F^i_k-\hat{F}^i_k(t)|, |F^i_{k-1}-\hat{F}^i_{k-1}(t)|)  \geq \varepsilon}\\
	 &= \prob{\max_{k\in[L]} |F^i_k-\hat{F}^i_k(t)|)  \geq \varepsilon}\\ 
	 &\leq 2\exp(-2n\varepsilon^2),
    \end{align*}
    where the last inequality holds from Lemma~\ref{lem:multinomial_cumulative_bound}.
\end{proof}

Using this lemma, we can derive the tail probability of the Borda score as follows.

\begin{lem}\label{lem:borda_score_bound}
    For all $t>0, i\in[K]$ and $n' \geq n > 0$, we have
    \begin{align*}
        &\prob{|B_i-\hat{B}_i(t)| \geq \varepsilon,  N_i(t) = n', \min_{j\in[K]} N_j(t) = n } \\
        &\quad~~~~~~~ \leq C_4\exp\left(- \frac{2nn'}{(\sqrt{n}+\sqrt{n'})^2}\varepsilon^2\right),
    \end{align*}
 where $C_4 = (2(K-1)+2^L)$.
\end{lem}

\begin{proof}
We denote $\varepsilon_1$ and $\varepsilon_2$ as 
\begin{align*}
 &\varepsilon_1 = \frac{2\sqrt{n}}{\sqrt{n'}+\sqrt{n}}\varepsilon,\\
 &\varepsilon_2 = \frac{\sqrt{n'}}{\sqrt{n'}+\sqrt{n}}\varepsilon. 
\end{align*}
We first show that for all $i\in[K]$ and $t>0$ we have
\begin{align*}
&|\arm{\hat{\vec{P}}}{i}(t) - \arm{\vec{P}}{i}|<\varepsilon_1, ~\forall j\neq i\,\max_{k\in[L]}\left|\arm{\hat{\vec{w}}}{j}(t) - \arm{\vec{w}}{j} \right| \leq \varepsilon_2\\
&\quad \Rightarrow |B_i-\hat{B}_i(t)| \leq \varepsilon.
\end{align*}
This can be shown by the simple calculation as follows.
\begin{align*}
    &\hat{B}_i(t) \\
    &= \frac{1}{K-1} \sum_{j\neq i} \mu(\arm{\hat{\vec{P}}}{i}(t),\arm{\hat{\vec{P}}}{j}(t))\\
    &= \frac{1}{K-1} \sum_{j\neq i} (\arm{\hat{\vec{P}}}{i}(t))^\top\arm{\hat{\vec{w}}}{j}(t)\\
    &= \frac{1}{K-1} \sum_{j\neq i}  (\arm{\hat{\vec{P}}}{i}(t))^\top(\arm{\hat{\vec{w}}}{j}(t)-\vecarm{w}{j} + \vecarm{w}{j})\\
    &\leq \varepsilon_2 + \frac{1}{K-1} \sum_{j\neq i} (\arm{\hat{\vec{P}}}{i}(t))^\top \vecarm{w}{j}\\
    &\quad\quad\quad(\because \max_{k\in[L]} \left|\arm{\hat{\vec{w}}}{j}(t) - \arm{\vec{w}}{j} \right| \leq \varepsilon_2)\\
    &= \varepsilon_2 +  \frac{1}{K-1} \sum_{j\neq i} \mu(\arm{\hat{\vec{P}}}{i}(t),\arm{\vec{P}}{j})\\
    &= B_i + \varepsilon_2 \\
    &\quad +\frac{1}{K-1} \sum_{j\neq i} \mu(\arm{\hat{\vec{P}}}{i}(t),\arm{\vec{P}}{j}) -\mu(\arm{\vec{P}}{i},\arm{\vec{P}}{j})\\
    &\leq B_i(t) + \varepsilon_2  + \varepsilon_1\\
    &\quad\quad\quad(\because |\arm{\hat{\vec{P}}}{i}(t) - \arm{\vec{P}}{i}|<\varepsilon_1~\text{and}~\text{Lemma~\ref{lem:characters-of-mu}})\\
    &= B_i(t) + \varepsilon.
\end{align*}
We can also show $ \hat{B}_i(t) \geq B_i(t) - \varepsilon$ in the same way. Therefore, from the union bound and Lemmas~\ref{lem:multinomial-l1-bound}~and~\ref{lem:w-bound}, we have
\begin{align*}
&\prob{|B_i-\hat{B}_i(t)| \geq \varepsilon,  N_i(t) = n', \min_j N_j(t) = n} \\
&\leq \mathbb{P}\left[|\arm{\hat{\vec{P}}}{i}(t) - \arm{\vec{P}}{i}|<\varepsilon_1,  N_i(t) = n', \min_j N_j(t) = n, \right.\\
&\quad~~~~~ \left.\forall j\neq i~~ \max_{k\in[L]}\left|\arm{\hat{\vec{w}}}{j}(t) - \arm{\vec{w}}{j} \right| \leq \varepsilon_2\right]\\
&\leq  \prob{|\vec{P}^i - \hat{\vec{P}}^i| \geq \varepsilon_1, N_i(t) = n'} \\
&\quad + \sum_{j\neq i} \prob{ \max_{k\in[L]} \left|\arm{\hat{\vec{w}}}{j}(t) - \arm{\vec{w}}{j} \right| \leq \varepsilon_2, N_j(t) \geq n}\\
&\leq 2^L \exp\left(-\frac12 n' \varepsilon_1^2\right) + (K-1)2\exp\left(-2 n \varepsilon^2_2\right)\\
&\leq  (2^L +2(K-1))\exp\left(- \frac{2nn'}{(\sqrt{n}+\sqrt{n'})^2}\varepsilon^2\right).
\end{align*}
\end{proof}

We bound the regret of UCB-Borda by using this concentration inequality of the Borda score. 
Let $N_{\mathrm{max}}(t)$ be the maximum number of times that sub-optimal arms are pulled, denoted as
\begin{align*}
N_{\mathrm{max}}(t) = \max_{i\neq{a^*_\Borda}} N_i(t).
\end{align*}
We define event $E_\mathrm{max}(t,n)$ as 
\begin{align*}
E_\mathrm{max}(t,n) = \{N_{\text{max}}(t) = n, N_{\text{max}}(t+1) = n+1\}
\end{align*}
so that we have
\begin{align*}
 \expect{N_\mathrm{max}(t)} = \sum_{t=1}^T\sum_{n=1}^t \prob{E_\mathrm{max}(t,n)}.
\end{align*}
We can decompose event $E_\mathrm{max}(t,n)$ as follows
\begin{align*}
	E_\mathrm{max}(t,n) &= E_A(t,n) \cup \left(\bigcup_{i\neq a^*_\Borda}\arm{E}{i}_U(t,n)\right),
\end{align*}
where $E_A(t,n)$ and $\arm{E}{i}_U(t,n)$ are defined as
\begin{align*}
	E_A(t, n) &= \left\{i_{\Count}(t) = \{a^*_\Borda\}, i_\UCB(t) \neq a^*_\Borda,\right.\\
	 &\quad ~~~~~~~\left.\forall i\neq a^*_\Borda~N_i(t) = n\right\},\\
	\arm{E}{i}_U(t,n) &= \left\{i \in i_{\text{Count}}(t), n_i(t) = n, i_{\text{UCB}}(t) = i\right\}.
\end{align*}
The number of having these events can be bounded by Lemmas~\ref{lem:Borda_UCB_sub_bound-1} and~\ref{lem:Borda_UCB_sub_bound-2}.

\begin{lem}\label{lem:Borda_UCB_sub_bound-1}
	For all $\varepsilon>0$ and $\alpha \geq 2$, we have
    \begin{align*}
    & \sum_{t=1}^T \sum_{n=1}^{t} \prob{E_A(t,n)}\\
    &\quad \leq \frac{4\alpha}{(\Delta^\Borda_{\mathrm{min}}-2\varepsilon)^2}\log T +\frac{2C_4(K-1)}{\varepsilon^2}\\
    &\quad\quad\quad\quad\quad\quad  +\frac{C_4}{1-\mathrm{e}^{-\frac12\varepsilon^2}} - \frac{C_4\log(1-\mathrm{e}^{-\frac12\varepsilon^2})}{\varepsilon^2}.
    \end{align*}
\end{lem}

\begin{proof}
	By denoting $\bar{B}_i(t)$ as $\bar{B}_i(t) = \hat{B}_i(t)+\beta_i(t)$, we have
\begin{align}
    &\sum_{t=1}^T \sum_{n=1}^{t} \prob{E_A(t,n)}\notag\\
    &\leq  \sum_{t=1}^T \sum_{n=1}^{t/K}\prob{E_A(t,n), \bar{B}_{i_{\text{UCB}}(t)}(t)\geq B_{a^*_\Borda}-\varepsilon}\notag\\
    &\quad + \sum_{t=1}^T \sum_{n=1}^{t/K} \prob{E_A(t,n), \bar{B}_{i_{\text{UCB}}(t)}(t) \leq B_{a^*_\Borda}-\varepsilon}.\label{eq:proof_regret_bound_UCB_tmp1}
\end{align}
Since $E_A(t,n)$ implies $N_{a^*_\Borda}(t) > n$, we have
\begin{align*}
    \beta_{i}(t) \leq 2\sqrt{\frac{\alpha}{n}\log t}
\end{align*}
for all $i\neq a^*_\Borda$.
Hence, if we define $L_\mathrm{max} = \frac{4\alpha}{(\Delta^\Borda_{\mathrm{min}}-2\varepsilon)^2}\log T$, we have
\begin{align*}
    &\sum_{t=1}^T \sum_{n=1}^{t} \prob{E_A(t,n), \bar{B}_{i_{\text{UCB}}(t)}(t) \geq B_{a^*_\Borda}-\varepsilon}\\
    &\leq \sum_{n=1}^T \prob{\bigcup_{t=1}^T \left\{ E_A(t,n), \bar{B}_{i_{\text{UCB}}(t)}(t) \geq B_{a^*_\Borda}-\varepsilon\right\} }\\
    &\leq  L_\mathrm{max} \\
    &\quad + \sum_{n=L_\mathrm{max}}^T \sum_{i\neq a^*_\Borda} \prob{\bigcup_{t=1}^T \left\{ E_A(t,n), \hat{B}_i(t) \geq B_i + \varepsilon\right\} }\\
    &= L_\mathrm{max} +  \sum_{n=L_\mathrm{max}}^T \sum_{i\neq a^*_\Borda} \prob{\hat{B}_i(\arm{\tau}{i}_n + 1) \geq B_i + \varepsilon}\\
    &\leq L_\mathrm{max} + (K-1)\sum_{n=L_\mathrm{max}}^T C_4\exp\left(-\frac{1}{2} n\varepsilon^2\right)\\
    &\leq  L_\mathrm{max} +  \frac{2C_4(K-1)}{\varepsilon^2}
\end{align*}
from Lemma~\ref{lem:borda_score_bound}, which is the upper-bound of the first term in \eqref{eq:proof_regret_bound_UCB_tmp1}. Here, the equation holds from the fact that $E_A(t;n)$ only occurs at most once in $t=1,\dots T$ and $\hat{B}_i(t)$ changes only when arm $i$ is pulled.

We now bound the second term in \eqref{eq:proof_regret_bound_UCB_tmp1}. Since $E_A(t,n)$ implies $N_{a^*_\Borda}(t) \geq n$, we have
\begin{align*}
    \beta_{a^*_\Borda}(t) \geq \sqrt{\frac{\alpha}{N_{a^*_\Borda}(t)}\log t}+\sqrt{\frac{\alpha}{n}\log t} \geq \sqrt{\frac{\alpha}{n}\log t}~.
\end{align*}
Therefore, we have
\begin{align*}
    &\sum_{t=1}^T \sum_{n=1}^{t/K} \indi{E_A(t,n), \bar{B}_{i_{\text{UCB}}(t)}(t) \leq B_{a^*_\Borda}-\varepsilon}\\
    &\leq \sum_{t=1}^T \sum_{n=1}^{t/K} \indi{E_A(t,n), \bar{B}_{a^*_\Borda}(t) \leq B_{a^*_\Borda}-\varepsilon}\\
    &= \sum_{t=1}^T \sum_{n=1}^{t/K} \indi{E_A(t,n), \beta_{a^*_\Borda}(t) \leq B_{a^*_\Borda}-\varepsilon - \hat{B}_{a^*_\Borda}(t)}\\
    &\leq \sum_{t=1}^T \sum_{n=1}^{t/K} \mathbbm{1}\left[E_A(t,n), \hat{B}_{a^*_\Borda}(t) \leq B_{a^*_\Borda}-\varepsilon,\right.\\
    &\quad\quad\quad\quad~~~~~~~~~ \left. t < \exp\left(\frac{n}{\alpha}({B}_{a^*_\Borda} - \varepsilon-\hat{B}_{a^*_\Borda}(t))^2 \right)\right]\\
    &\leq \sum_{n=1}^{T/K}  \exp\left(\frac{n}{\alpha}\left({B}_{a^*_\Borda} - \varepsilon-\hat{B}_{a^*_\Borda}\left(\arm{\tau}{a^*_\Borda}_n+1\right)\right)^2 \right)\\
    &\quad\quad ~~~~~~~~~~~~~~~\indi{ \bigcup_{t=1}^T \left\{E_A(t,n), \hat{B}_{a^*_\Borda}(t) \leq B_{a^*_\Borda}-\varepsilon\right\}}\\
   &\leq \sum_{n=1}^{T/K}  \exp\left(\frac{n}{\alpha}\left({B}_{a^*_\Borda} - \varepsilon-\hat{B}_{a^*_\Borda}\left(\arm{\tau}{a^*_\Borda}_n+1\right)\right)^2 \right)\\
    &\quad\quad ~~~~~~~~~~~~~~~\indi{ \hat{B}_{a^*_\Borda}\left(\arm{\tau}{a^*_\Borda}_n+1\right) \leq B_{a^*_\Borda}-\varepsilon}.
\end{align*}

Now, since 
\begin{align*}
&\prob{ \hat{B}_{a^*_\Borda}\left(\arm{\tau}{a^*_\Borda}_n+1\right) \leq x} \\
&~~~~~~~~~~\leq C_4\exp\left(-\frac12 n (B_{a^*_\Borda}-x)^2 \right)
\end{align*}
holds from Lemma~\ref{lem:borda_score_bound}, by integration by parts we have
\begin{align*}
    &\expect{\sum_{t=1}^T \sum_{n=1}^{t/K} \indi{E_A(t,n), \bar{B}_{i_{\text{UCB}}(t)}(t) \leq B_{a^*_\Borda}-\varepsilon}}\\
    &\leq \sum_{n=1}^{T/K}\mathbb{E}\left[\exp\left(\frac{n}{\alpha}\left({B}_{a^*_\Borda} - \varepsilon-\hat{B}_{a^*_\Borda}\left(\arm{\tau}{a^*_\Borda}_n+1\right)\right)^2 \right) \right.\\
    &\quad~~~~~~~~~~ \left.\indi{ \hat{B}_{a^*_\Borda}\left(\arm{\tau}{a^*_\Borda}_n+1\right) \leq B_{a^*_\Borda}-\varepsilon}\right]\\
    &= \sum_{n=1}^{T/K} \int_0^{{B}_{a^*_\Borda} - \varepsilon} \exp\left(\frac{n}{\alpha}\left({B}_{a^*_\Borda} - \varepsilon-x\right)^2 \right) \\
    &\quad~~~~~~~~~~ \prob{\hat{B}_{a^*_\Borda}\left(\arm{\tau}{a^*_\Borda}_n+1\right) = x} \intd x\\
    &\leq \sum_{n=1}^{T/K}  C_4\int_{-\infty}^{{B}_{a^*_\Borda} - \varepsilon} \exp\left(\frac{n}{\alpha}({B}_{a^*_\Borda}- \varepsilon-x)^2 \right)\\
    &\quad\quad\quad\quad~~~~~~~~~ \exp\left(-\frac{n}2 (B_{a^*_\Borda}-x-\varepsilon)^2 \right) \intd x\\
    &\leq \sum_{n=1}^{T/K}  C_4\exp\left(-\frac12 n \varepsilon^2 \right)\\
    &\quad+ C_4\int_{-\infty}^{{B}_{a^*_\Borda} - \varepsilon}\left\{\frac{\intd}{\intd x}\left(-\exp\left(\frac{n}{\alpha}({B}_{a^*_\Borda}- \varepsilon-x)^2 \right)\right)\right\}\\
    &\quad\quad\quad\quad~~~~~~~~~~~~~~~~~~~ \exp\left(-\frac12 n (B_{a^*_\Borda}-x)^2 \right) \intd x\\
    &\leq \sum_{n=1}^{T/K} C_4\exp\left(-\frac12 n \varepsilon^2 \right)\\
	&\quad + \frac{2nC_4}{\alpha}\int_{-\infty}^{{B}_{a^*_\Borda} - \varepsilon}({B}'_{a^*_\Borda} - x)\\
	&\quad~~~~\exp\left(-\frac12 n (B_{a^*_\Borda}-x)^2+\frac{n}{\alpha}(B_{a^*_\Borda}-x+\varepsilon)^2 \right) \intd x.
\end{align*}
Therefore, if $\alpha\geq2$, the regret is bounded as
\begin{align*}
	&\expect{\sum_{t=1}^T \sum_{n=1}^{t/K} \indi{E_A(t,n), \bar{B}_{i_{\text{UCB}}(t)}(t) \leq B_{a^*_\Borda}-\varepsilon}}\\
	    &= \sum_{n=1}^{T/K}  C_4\exp\left(-\frac12 n \varepsilon^2 \right)+ C_4n\mathrm{e}^{-\frac12n\varepsilon^2}\\
    &\quad~~~~~  \int_{-\infty}^{{B}_{a^*_\Borda} - \varepsilon}({B}_{a^*_\Borda} - x - \varepsilon)\mathrm{e}^{\left(n \varepsilon\left(x-{B}_{a^*_\Borda} + \varepsilon\right)\right)} \intd x\\
    &= \sum_{n=1}^{T/K}  C_4\mathrm{e}^{-\frac12n\varepsilon^2}\left(1+\frac1{n\varepsilon^2}\right)\\
    & \leq \frac{C_4}{1-\mathrm{e}^{-\frac12\varepsilon^2}} - \frac{C_4\log(1-\mathrm{e}^{-\frac12\varepsilon^2})}{\varepsilon^2}.
\end{align*}
Combining all the discussion above, we have
\begin{align*}
    &\expect{ \sum_{t=1}^T \sum_{n=1}^{t} \indi{A_n}}\\
    &\quad \leq \frac{4\alpha}{(\Delta^\Borda_{\mathrm{min}}-2\varepsilon)^2}\log T +\frac{2C_4(K-1)}{\varepsilon^2}\\
    &\quad\quad\quad\quad\quad\quad  +\frac{C_4}{1-\mathrm{e}^{-\frac12\varepsilon^2}} - \frac{C_4\log(1-\mathrm{e}^{-\frac12\varepsilon^2})}{\varepsilon^2}.
\end{align*}
\end{proof}

\begin{lem}\label{lem:Borda_UCB_sub_bound-2}
If $\alpha$ satisfies
\begin{align*}
\alpha \geq \frac{3(1+3\varepsilon')^2}{2(1-\varepsilon')^2}\left(\frac{K-1}{K-2}\right)^2 
\end{align*}
for $\varepsilon'>0$, we have
    \begin{align*}
    & \sum_{t=1}^T \sum_{n=1}^{t} \prob{\arm{E}{i}_{U}(t,n)}\\
    & \leq 2C_4(K-1)(\varepsilon')^2\frac{\pi^2}{6} \\
    &\quad+   \frac{2C_4}{\left(1-\exp\left(-\frac{(\Delta^\Borda_{i})^2}{8K-1}\right)\right)\left(1-\exp\left(-\frac{(\varepsilon')^2(\Delta^\Borda_{i})^2}{8}\right)\right)}.
    \end{align*}
\end{lem}

\begin{proof}
 We have
\begin{align*}
    &\sum_{t=1}^T \sum_{n=1}^{t} \indi{\arm{E}{i}_U(t,n)}\\
    &\leq \sum_{t=1}^T \sum_{n=1}^{t} \indi{i \in i_{\text{Count}}(t), T_i(t) = n, \bar{B}_i(t) > \bar{B}_{a^*_\Borda}(t)}\\
    &\leq \sum_{t=1}^T \sum_{n=1}^{t} \mathbbm{1}\left[\arm{E}{i}_{U'}(t,n),\Delta\beta_{i}(t) > \hat{B}_{a^*_\Borda}(t)-\hat{B}_{i}(t)\right],
\end{align*}
where $\arm{E}{i}_{U'}(t,n) = \{i \in i_{\text{Count}}(t), N_i(t) = n\}$ and
\begin{align*}
    \Delta\beta_{i}(t)&= \beta_{i}(t)-\beta_{a^*_\Borda}(t)\\
    &=\frac{K-2}{K-1}\left(\sqrt{\frac{\alpha}{N_j(t)}\log t}- \sqrt{\frac{\alpha}{N_{a^*_\Borda}(t)}\log t}\right).  
\end{align*}
In the following, we use function $\Delta\beta(n,m)$ given by 
\begin{align*}
    \Delta\beta(n,m)=\frac{K-2}{K-1}\left(\sqrt{\frac{\alpha}{n}\log t}- \sqrt{\frac{\alpha}{m}\log t}\right)  
\end{align*}
for notational simplicity. Now, since $\arm{E}{i}_{U'}(t,n)$ implies $N_{a^*_\Borda}(t) = n' = \frac{t-n}{K-1}$, we have
\begin{align*}
    &\prob{\arm{E}{i}_{U'}(t,n), \Delta\beta_{i}(t) > \hat{B}_{a^*_\Borda}(t)-\hat{B}_{i}(t)}\\
    &\leq \prob{\arm{E}{i}_{U'}(t,n), |\hat{B}_{a^*_\Borda}(t) - B_{a^*_\Borda}| > \lambda\left(\Delta^\Borda_i + \Delta\beta_{i}(t))\right)}\\
    &\quad + \prob{\arm{E}{i}_{U'}(t,n),|\hat{B}_i(t) - B_i| > (1-\lambda)\left(\Delta^\Borda_i + \Delta\beta_{i}(t)\right)}\\
    & \leq C_4\exp\left(- \frac{2nn'}{(\sqrt{n}+\sqrt{n'})^2}\lambda^2\left(\Delta^\Borda_i + \Delta\beta(n,n')\right)^2\right)\\
    &\quad + C_4\exp\left(- \frac{n'}{2}(1-\lambda)^2\left(\Delta^\Borda_i + \Delta\beta(n,n')\right)^2\right)
\end{align*}
for all $0\leq \lambda \leq 1$ from the union bound and Lemma~\ref{lem:borda_score_bound}. Setting $\lambda$ to
\begin{align*}
    \lambda = \frac{\sqrt{n'/2}}{\sqrt{n'/2} + \frac{\sqrt{2nn'}}{\sqrt{n}+\sqrt{n'}}}
\end{align*}
yields
\begin{align*}
    &\prob{\arm{D'}{i}, \hat{B}_j(t) - \hat{B}_{a^*_\Borda}(t) > \Delta\beta_t(t)}\\
    &\leq 2C_4\exp\left(- \frac{2nn'}{(\sqrt{n}+\sqrt{n'})^2}\left(\Delta^\Borda_i + \Delta\beta(n,n')\right)^2\right)\\
    &\leq 2C_4\exp\left(- \frac{n'}{8}(\Delta_i^\Borda)^2 \right.\\
    &\quad~~~~~~~~~~~~~~~~~\left.- \frac{2(\sqrt{n}-\sqrt{n'})^2}{(\sqrt{n'}+3\sqrt{n})^2}\left(\frac{K-2}{K-1}\right)^2\alpha \log t\right),
\end{align*}
where the last inequality holds for $n'<n$.  For $T'(n) = \floor{(1+(K-1)(\varepsilon')^2)n}$, we bound the probability of having event $\arm{E}{i}_U(t,n)$ in the case of  $t \leq T'(n)$ and $t \geq T'(n)+1$ separately. For the case of  $t \leq T'(n)$, we have
\begin{align*}
	 & \sum_{n=1}^{T} \sum_{t=1}^{T'(n)} \prob{\arm{E}{i}_U(t,n)}\\
	 &\leq  \sum_{n=1}^{T}\sum_{t=n}^{T'(n)} \prob{\arm{E}{i}_{U'}(t,n), \Delta\beta_i(t) > \hat{B}_{a^*_\Borda}(t) - \hat{B}_{i}(t)}\\
	 &\leq  \sum_{n=1}^{T}\sum_{t=n}^{T'(n)} 2C_4\exp\left( - \frac{2(\sqrt{n}-\sqrt{n'})^2}{(\sqrt{n'}+3\sqrt{n})^2}\left(\frac{K-2}{K-1}\right)^2\alpha \log t\right)\\
	  &\leq \sum_{n=1}^{T}\sum_{t=n}^{T'(n)} \frac{2C_4}{t^3}~~(\text{by}~ t \leq T'(n) \Rightarrow n' \leq \varepsilon'n)\\
	  &\leq \sum_{n=1}^{T} (T'(n)-n) \frac{2C_4}{n^3}\\
	  &\leq 2C_4(K-1)(\varepsilon')^2\frac{\pi^2}{6}.
\end{align*}
Moreover, the following holds when $t \geq T'(n)+1$.
\begin{align*}
	&\sum_{n=1}^{T}\sum_{t=T'(n)+1}^{T} \prob{\arm{E}{i}_U(t,n)}\\
	&\leq \sum_{n=1}^{T}\sum_{t=T'(n)+1}^{T} \prob{\arm{E}{i}_{U'}(t,n), \Delta\beta_i(t) > \hat{B}_{a^*_\Borda}(t) - \hat{B}_{i}(t)}\\
	&\leq \sum_{n=1}^{T}\sum_{t=T'(n)+1}^{T} 2C_4\exp\left(- \frac{\frac{t-n}{K-1}}{8}(\Delta^\Borda_{i})^2\right)\\
	&\leq \sum_{n=1}^{T} 2C_4\frac{1}{1-\exp\left(-\frac{(\Delta^\Borda_{i})^2}{8K-1}\right)}\exp\left(- \frac{n(\varepsilon')^2}{8}(\Delta^\Borda_{i})^2\right)\\
	&\leq 2C_4\frac{1}{1-\exp\left(-\frac{(\Delta^\Borda_{i})^2}{8K-1}\right)}\frac{1}{1-\exp\left(-\frac{(\varepsilon')^2(\Delta^\Borda_{i})^2}{8}\right)}.
\end{align*}
Therefore, we have
\begin{align*}
    &\expect{\sum_{t=1}^T \sum_{n=1}^{t} \indi{\arm{E}{i}_U(t,n)}}\\
   &\leq  (K-1)(\varepsilon')^2\frac{\pi^2}{6} \\
    &\quad+   2C_4\frac{1}{1-\exp\left(-\frac{(\Delta^\Borda_{i})^2}{8K-1}\right)}\frac{1}{1-\exp\left(-\frac{(\varepsilon')^2(\Delta^\Borda_{i})^2}{8}\right)}.
\end{align*}
\end{proof}

By combining Lemma~\ref{lem:Borda_UCB_sub_bound-1} and \ref{lem:Borda_UCB_sub_bound-2}, we derive the regret for Borda-UCB.
\begin{proof}[Proof of Theorem~\ref{thm:UCB-borda-regret}]
Let $\Delta^\Borda_{\mathrm{all}} = \sum_{i\neq a^*_\Borda} \Delta^\Borda_i$.
Then, the regret is bounded as 
\begin{align*}
    &\expect{R_\Borda(T)}\\
    &\quad\leq \Delta^\Borda_{\mathrm{all}} \expect{N_{\text{max}}(T)}\\
    &\quad= \Delta^\Borda_{\mathrm{all}} \sum_{t=1}^T \sum_{n=1}^{t} \prob{E_\mathrm{max}(t,n)}\\
    &\quad= \Delta^\Borda_{\mathrm{all}} \sum_{t=1}^T \sum_{n=1}^{t} \prob{E_A(t,n)} \\
    &\quad~~~ +\Delta^\Borda_{\mathrm{all}}\sum_{t=1}^T \sum_{n=1}^{t}\sum_{i\neq a^*_\Borda} \prob{\arm{E}{i}_U(t,n)}.
\end{align*}
From Lemma~\ref{lem:Borda_UCB_sub_bound-1} and \ref{lem:Borda_UCB_sub_bound-2}, we have
\begin{align*}
    &\expect{R_\Borda(T)} \\
    & \leq  \Delta^\Borda_{\mathrm{all}}\left(\frac{4\alpha}{(\Delta^\Borda_{\text{min}}-2\varepsilon)^2}\log T + C_\varepsilon + C_{\varepsilon'} \right),
\end{align*}
where $C_\varepsilon$ and $C_{\varepsilon'}$ are defined as
\begin{align*}
    C_\varepsilon &= \frac{2C_4(K-1)}{\varepsilon^2}+\frac{C_4}{1-\mathrm{e}^{-\frac12\varepsilon^2}} - \frac{C_4\log(1-\mathrm{e}^{-\frac12\varepsilon^2})}{\varepsilon^2},\\
    C_{\varepsilon'} &= \sum_{i\neq a^*_\Borda}  \frac{2C_4}{1-\exp\left(-\frac{(\Delta^\Borda_{i})^2}{8K-1}\right)}\frac{1}{1-\exp\left(-\frac{(\varepsilon')^2(\Delta^\Borda_{i})^2}{8}\right)}\\
    &\quad~~~~~~~~~~~~~~~~~~ + 2C_4(K-1)^2(\varepsilon')^2\frac{\pi^2}{6}.
\end{align*}
\end{proof}

\section{Proof of Theorem~\ref{thm:borda-lower-bound}} \label{sec:Proof-of-borda-lower-bound}
In this section, we prove Theorem~\ref{thm:borda-lower-bound}, which is inspired by \citet{Lai1985}.

\begin{proof}[Proof of Theorem~\ref{thm:borda-lower-bound}]
We prove the theorem by constructing the pair of instances $(\Gamma, \Theta)$, which consists of three feedback distributions denoted as $\Gamma = (\vecarm{P}{1}_\Gamma, \vecarm{P}{2}_\Gamma, \vecarm{P}{3}_\Gamma)$ and $\Theta = (\vecarm{P}{1}_\Theta, \vecarm{P}{2}_\Theta, \vecarm{P}{3}_\Theta)$. We show that the statement of the theorem is satisfied for the distributions given by
\begin{align*}
&\vecarm{P}{1}_\Gamma = \vecarm{P}{1}_\Theta = \left(\varepsilon,\varepsilon,1-4\varepsilon,\varepsilon,\varepsilon\right)^\top,\\
&\vecarm{P}{2}_\Gamma = \vecarm{P}{2}_\Theta = \left(\varepsilon,\frac14+\varepsilon,\frac12-4\varepsilon,\frac14 +\varepsilon,\varepsilon\right)^\top,\\
&\vecarm{P}{3}_\Theta = \left(\frac12,\frac14, \varepsilon, \frac12-2\varepsilon,\varepsilon\right)^\top,\\
&\vecarm{P}{3}_\Gamma = \left(\frac12,\frac14-2\varepsilon, \varepsilon, \frac12,\varepsilon\right)^\top.
\end{align*}
for $\varepsilon \in (0,1/8)$. 

First, using Markov inequality, we have
\begin{align}
&\mathbb{P}_\Gamma\left[N_3(T)\KL{\vecarm{P}{3}_\Gamma}{\vecarm{P}{3}_\Theta} \leq (1-\delta)\log T\right]\\
&\leq \mathbb{P}_\Gamma\left[T- N_3(T)\KL{\vecarm{P}{3}_\Gamma}{\vecarm{P}{3}_\Theta} \geq T-(1-\delta)\log T\right]\notag\\
&\leq \frac{1}{T-(1-\delta)\log T}\mathbb{E}_\Gamma\left[t- N_3(T)\KL{\vecarm{P}{3}_\Gamma}{\vecarm{P}{3}_\Theta} \right]\notag\\
&= o(T^{a-1}) \label{eq:proof-borda-lower-bound-tmp1}
\end{align}
for $0< a< \delta < 1$, where $\mathbb{P}_\Gamma$ and $\mathbb{E}_\Gamma$ are the probability and expectation in the instance $\Gamma$, respectively. The last equality holds from the assumption that algorithm achieves sub-polynomial regret in $\Gamma$. Let $Y_1, Y_2,\dots$ be the successive feedback observed from pulling arm $3$ in $\Gamma$.  We furthermore define the likelihood ratio $\mathcal{L}_m$ as
\begin{align*}
	\mathcal{L}_m = \sum_{n=1}^m \log\frac{(\vecarm{P}{3}_\Theta)_{Y_n}}{(\vecarm{P}{3}_\Gamma)_{Y_n}},
\end{align*}
where $(\vecarm{P}{3}_\Theta)_{i}$ and $(\vecarm{P}{3}_\Gamma)_{i}$ means the $i$-th element of $\vecarm{P}{3}_\Theta$ and $\vecarm{P}{3}_\Gamma$, respectively. Let events $G_T, G'_T(n_1,n_2,n_3)$ be 
\begin{align*}
	&G_T = \left\{N_3(T) \leq \frac{(1-\delta)\log T}{\KL{\vecarm{P}{3}_\Theta}{\vecarm{P}{3}_\Gamma}}, \mathcal{L}_m \leq (1-a)\log T\right\},\\
	&G'_T(n_1,n_2,n_3) = \{N_1(T) = n_1,N_2(T) = n_2,N_3(T) = n_3,\\
	&\quad\quad\quad~~~~~~~~~~~~~~ \mathcal{L}_m \leq (1-a)\log T\}.
\end{align*}
Then, $\mathbb{P}_\Gamma\left[G_T\right] =  o(T^{a-1})$ holds directly from \eqref{eq:proof-borda-lower-bound-tmp1}. Note that
\begin{align*}
	&\mathbb{P}_\Gamma[G'_T(n_1,n_2,n_3)]\\
	&= \mathbb{E}_{\Theta}\left[ \left(\sum_{n=1}^m \frac{(\vecarm{P}{3}_\Gamma)_{Y_n}}{(\vecarm{P}{3}_\Theta)_{Y_n}} \right)\indi{G'_T(n_1,n_2,n_3)}\right]\\
	&\leq \exp(-(1-a)\log T)\mathbb{P}_\Theta \left[G'_T(n_1,n_2,n_3)\right],
\end{align*}
where $\mathbb{P}_\Theta$ is the probability under $\Theta$. Since 
$G_T$ is the disjoint union of $\{G'_T(n_1,n_2,n_3)\}_{n_1+n_2+n_3=T}$, we have
\begin{align}
	\mathbb{P}_\Theta[G_T] \leq T^{1-a}\mathbb{P}_\Gamma[G_T] \to 0~\text{as}~T\to\infty. \label{eq:proof-borda-lower-bound-tmp2}
\end{align}
Furthermore, using the strong law of large numbers, we have 
\begin{align*}
	\mathcal{L}_m/m \to\KL{\vecarm{P}{3}_\Theta}{\vecarm{P}{3}_\Gamma}
\end{align*}
almost surely on $\mathbb{P}_\Theta$. Hence, using $1-a > 1-\delta$, we have
\begin{align}
	&\mathbb{P}_\Theta\left[\exists m< \frac{(1-\delta)\log T}{\KL{\vecarm{P}{3}_\Theta}{\vecarm{P}{3}_\Gamma}}~~~\mathcal{L}_m \geq (1-a)\log T\right]\notag\\
	&\quad\quad\quad\quad \to 0 \quad \text{as}~T\to \infty,\label{eq:proof-borda-lower-bound-tmp3}
\end{align}
and thus, from \eqref{eq:proof-borda-lower-bound-tmp2} and \eqref{eq:proof-borda-lower-bound-tmp3}, we have
\begin{align*}
	&\mathbb{P}_\Theta\left[N_3(T) \leq \frac{(1-\delta)\log T}{\KL{\vecarm{P}{3}_\Theta}{\vecarm{P}{3}_\Gamma}}\right]\\
	&\quad\quad\quad\quad \to 0 \quad \text{as}~T\to \infty,
\end{align*}
which implies
\begin{align*}
	&\liminf_{T\to\infty} \frac{\mathbb{E}_\Theta\left[N_3(T)\right]}{\log T}  \geq \frac{(1-\delta)}{\KL{\vecarm{P}{3}_\Theta}{\vecarm{P}{3}_\Gamma}}.
\end{align*}
By simple calculation, we have 
\begin{align*}
	&\Delta^\Borda_2 = \Delta^\Borda_\mathrm{min} = \frac18 \varepsilon,\\
	&\Delta^\Borda_3 = \frac38 - \frac14\varepsilon + \frac{15}4\varepsilon^2 \geq \frac{81}{256},
\end{align*}
and we have
\begin{align*}
	&\KL{\vecarm{P}{3}_\Theta}{\vecarm{P}{3}_\Gamma} \leq \frac{4\varepsilon^2}{\frac14-2\varepsilon}
\end{align*}
since $\vec{P}'$ in \eqref{eq:P'-P3-KL} is equal to $\vecarm{P}{3}_\Theta$.
Therefore, 
\begin{align*}
	\liminf_{T\to\infty} \frac{\expect{R^\Theta_T}}{\log T} 
	&\geq \liminf_{T\to\infty} \Delta^\Borda_3\frac{\mathbb{E}_\Theta\left[N_3(T)\right]}{\log T}\\
	&\geq O\left(\frac{1}{\varepsilon^2}\right) = O\left(\frac{1}{(\Delta^\Borda_\mathrm{min})^2}\right),
\end{align*}
and hence the instances $\Gamma$ and $\Theta$ satisfy the statement of the theorem.
\end{proof}

\end{document}